%% file: main.tex
\PassOptionsToPackage{square,comma,numbers,sort&compress}{natbib}
\documentclass[10pt,twocolumn,letterpaper]{article}

\usepackage{cvpr}      \usepackage[accsupp]{axessibility}        

\input{preamble}

\definecolor{cvprblue}{rgb}{0.21,0.49,0.74}
\usepackage[pagebackref,breaklinks,colorlinks,citecolor=cvprblue]{hyperref}
\usepackage[linesnumbered,ruled,vlined]{algorithm2e}
\usepackage{amsthm}
 
\newtheorem{assumption}{Assumption}
\newtheorem{theorem}{Theorem}
\newtheorem{lemma}{Lemma}
\newtheorem{corollary}{Corollary}

\usepackage{amsthm,tgcursor,courier}
 \newcommand{\norm}{\texttt{norm}}

\usepackage{color}

\def\paperID{5575} 
\def\confName{CVPR}
\def\confYear{2024}

\title{Friendly Sharpness-Aware Minimization}

\author{Tao Li$^1$ \qquad Pan Zhou$^2$ \qquad Zhengbao He$^1$ \qquad Xinwen Cheng$^1$ \qquad Xiaolin Huang$^1$\\
$^1$Department of Automation,
Shanghai Jiao Tong University, Shanghai, China\\
$^2$School of Computing and Information Systems, Singapore Management University, Singapore \\
{\tt\small \{li.tao,lstefanie,xinwencheng,xiaolinhuang\}@sjtu.edu.cn \quad  panzhou@smu.edu.sg}
}

\begin{document}
\maketitle
\input{sec/0_abstract}    
\input{sec/1_intro}
{
    \small
   \bibliographystyle{unsrt}
    \bibliography{main}
}

\clearpage 
\newpage
\input{sec/X_suppl}


\end{document}

%% file: preamble.tex
\usepackage[dvipsnames]{xcolor}


%% file: sec/0_abstract.tex
\begin{abstract}

Sharpness-Aware Minimization (SAM) 
has been instrumental in improving deep neural network training by minimizing both training loss and loss sharpness. Despite the practical success, the mechanisms behind SAM's generalization enhancements remain elusive, limiting its progress in deep learning optimization. In this work, we investigate SAM's core components for generalization improvement and introduce ``Friendly-SAM'' (F-SAM) to further enhance SAM's generalization. Our investigation reveals the key role of batch-specific stochastic gradient noise within the adversarial perturbation, i.e., the current minibatch gradient, which significantly influences SAM's generalization performance. By decomposing the adversarial perturbation in SAM into full gradient and stochastic gradient noise components, we discover that relying solely on the full gradient component degrades generalization while excluding it leads to improved performance.   
The possible reason lies in the full gradient component's increase in sharpness loss for the entire dataset, creating inconsistencies with the subsequent sharpness minimization step solely on the current minibatch data. Inspired by these insights, F-SAM aims to mitigate the negative effects of the full gradient component. It removes the full gradient estimated by an exponentially moving average (EMA) of historical stochastic gradients, and then leverages stochastic gradient noise for improved generalization. Moreover, we provide theoretical validation for the EMA approximation and prove the convergence of F-SAM  on non-convex problems.  Extensive experiments demonstrate the superior generalization performance and robustness of F-SAM over vanilla SAM. Code is available at {\url{https://github.com/nblt/F-SAM}}.  
\end{abstract}

%% file: sec/1_intro.tex
\section{Introduction}
\label{sec:intro}

\begin{figure}[t]
	\centering
	\includegraphics[width=0.50\linewidth]{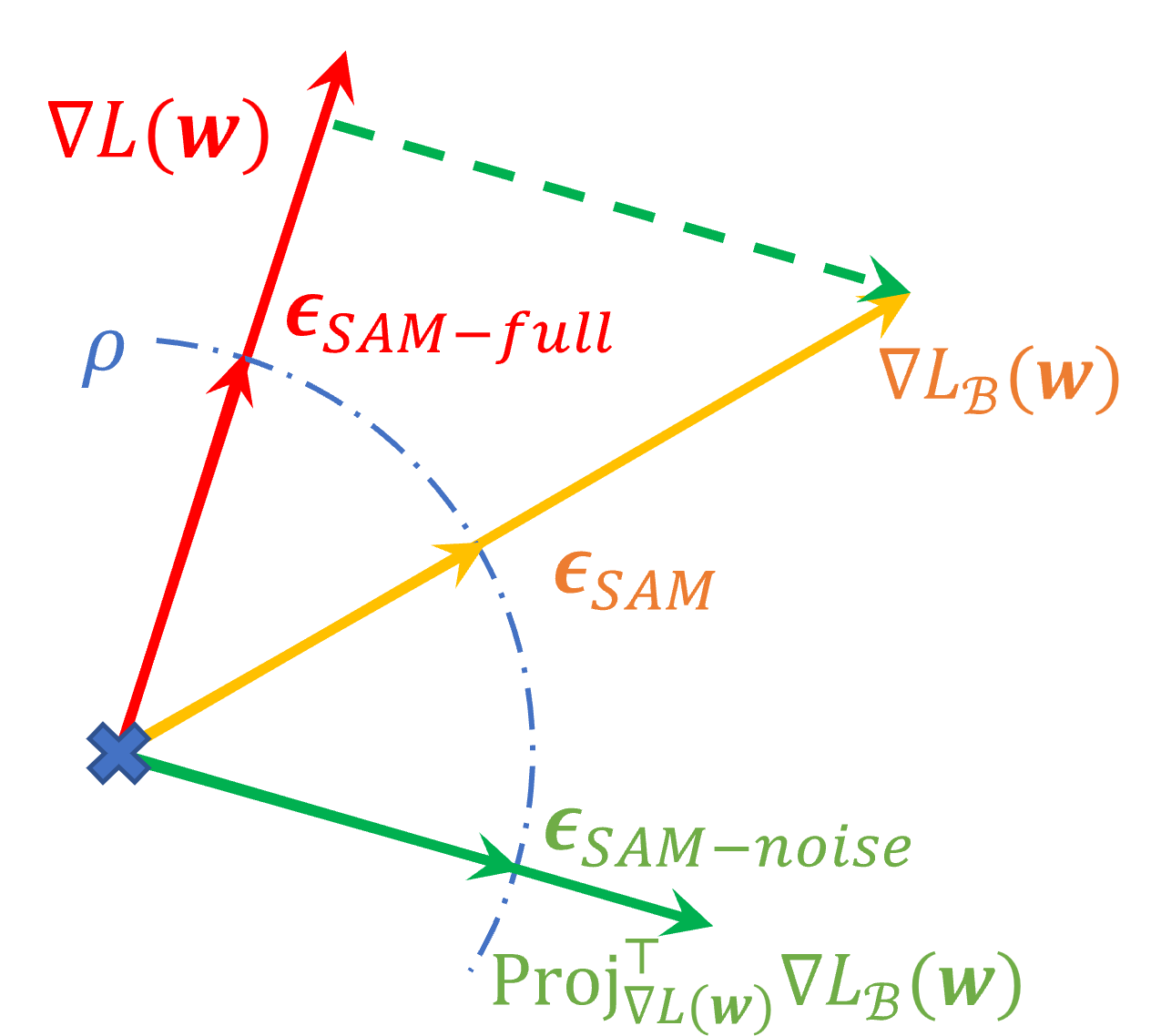}
	\includegraphics[width=0.49\linewidth]{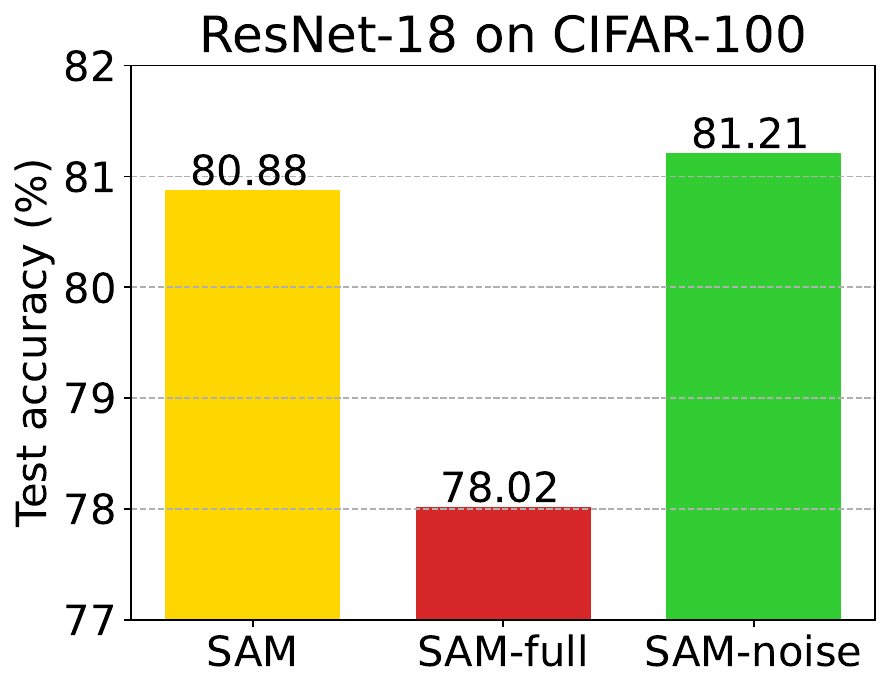}
	\caption{Investigation on SAM's adversarial  perturbation direction. We decompose the minibatch gradient $\nabla_\mathcal{B} L(\boldsymbol{w})$ into two  components: the full gradient component and the remaining batch-specific stochastic gradient noise.  Solely using the full gradient component leads to a dramatic generalization degradation, while only using the noise component enhances the generalization.
 }
 \vspace{-5mm}
	\label{fig:illustration}
\end{figure}

Deep neural networks (DNNs) have achieved remarkable performance in various vision and language processing tasks \cite{devlin2018bert, brown2020language, liu2021swin}. A critical factor contributing to their success is the choice of optimization algorithms \cite{robbins1951stochastic,kingma2014adam,loshchilov2017decoupled,zhuang2020adabelief,xie2022adan,zhou2022win}, designed to efficiently optimize DNN model parameters. To achieve better performance, it is often desirable for an optimizer to converge to flat minima, characterized by uniformly low loss values, as opposed to sharp minima with abrupt loss changes, as the former typically results in better generalization \cite{hochreiter1997flat,dinh2017sharp, li2018visualizing,zhou2020towards}. However, the complex landscape of over-parameterized DNNs, featuring numerous sharp minima, presents challenges for optimizers in real applications.

To address this challenge, Sharpness-Aware Minimization (SAM) \cite{foret2020sharpness} is proposed to simultaneously minimize the training loss  and the loss sharpness, enabling it to identify flat minima associated with improved generalization performance.  For each training iteration, given the minibatch training loss   $L_{\mathcal{B}}(\boldsymbol{w})$ parameterized by  $\boldsymbol{w}$, SAM achieves this by adversarially computing a adversarial perturbation  $\boldsymbol{\epsilon}_s$ to maximize  the   training loss  $L_{\mathcal{B}}(\boldsymbol{w}+\boldsymbol{\epsilon}_s)$. It then minimizes the loss of this perturbed objective via one updating step of a base optimizer, such as SGD~\citep{robbins1951stochastic}. In practice, to efficiently compute  $\boldsymbol{\epsilon}_s$,  SAM takes a linear approximation of the loss objective, and employs the minibatch gradient  $\nabla L_{\mathcal{B}}(\boldsymbol{w})$ as the accent direction for searching $\boldsymbol{\epsilon}_s = \rho\cdot \norm(\nabla L_{\mathcal{B}}(\boldsymbol{w}))$, where $\norm(\boldsymbol{x})= \boldsymbol{x}/ \|\boldsymbol{x}\|_2$ and $\rho>0$ is a radius. This approach enables SAM to seek models situated in neighborhoods characterized by consistently low loss values, resulting in better generalization when used in conjunction with base optimizers such as SGD and Adam~\cite{kingma2014adam}.

{Despite SAM's practical success, the underlying mechanism responsible for its empirical generalization improvements remains limited} \cite{andriushchenko2022towards,chen2023does}.
This open challenge hinders developing new and more advanced deep learning optimization algorithms  in a principle way.  
To address this issue, we undertake a comprehensive exploration of SAM's core components contributing to its generalization improvement. Subsequently, we introduce a new variant of SAM, which offers a simple  yet  effective approach to further enhance the generalization  performance  in conjunction with popular base optimizers.

First, our investigations reveal that the batch-specific stochastic gradient noise present in the minibatch gradient $\nabla L_{\mathcal{B}}(\boldsymbol{w})$ plays a crucial role in SAM's generalization performance. Specifically,  as shown in Fig.~\ref{fig:illustration}, we decompose the perturbation direction $\nabla L_{\mathcal{B}}(\boldsymbol{w})$ into two orthogonal components: the full gradient component $\mathrm{Proj}_{\nabla L(\boldsymbol{w})} \nabla L_{\mathcal{B}}(\boldsymbol{w})$, representing the projection of $\nabla L_{\mathcal{B}}(\boldsymbol{w})$ onto the full gradient direction $\nabla L(\boldsymbol{w})$, and the stochastic gradient noise component $\mathrm{Proj}^\top_{\nabla L(\boldsymbol{w})} \nabla L_{\mathcal{B}}(\boldsymbol{w})$, denoting the residual projection. Surprisingly, we empirically find that as shown in Fig.~\ref{fig:illustration},  using only the full gradient component for perturbation significantly degrades SAM's generalization, which is also observed by \cite{andriushchenko2022towards}.
Conversely, excluding the full gradient component leads to improved generalization performance. This observation suggests that the effectiveness of SAM primarily arises from the presence of the stochastic gradient noise component within the minibatch gradient $\nabla L_{\mathcal{B}}(\boldsymbol{w})$. The reason behind this phenomenon may be that the full gradient component in the perturbation increases the sharpness loss of the entire dataset, leading to inconsistencies with the subsequent sharpness minimization step which only uses the current minibatch data to minimize training loss value and its sharpness.   See more discussion in Sec.~\ref{sec:investigation}.

Next, building on this insight, we introduce a straightforward yet highly effective modification to SAM, referred to as ``Friendly-SAM" or F-SAM. F-SAM mitigates the undesirable effects of the full gradient component within the minibatch gradient and leverages the beneficial role of batch-specific stochastic gradient noise to further enhance SAM's generalization performance.  Since removing the full gradient component in the perturbation step can relieve the adversarial perturbation impact on other data points except for current minibatch data, the adversarial perturbation in F-SAM is more ``friendly'' to other data points compared with vanilla SAM, which helps the adversarial perturbation and the following minimization perform on the same minibatch data and improves sharpness minimization consistency.  Moreover, such friendliness enables F-SAM to be less sensitive to the choice of perturbation radius as discussed in \cref{sec:ablation}.  For efficiency,  F-SAM approximates the computationally expensive full gradient through an exponentially moving average (EMA) of the historical stochastic gradients and computes the stochastic gradient noise as the perturbation direction.  
Moreover, we also provide theoretical validation of the EMA approximation to the full gradient and prove the convergence of F-SAM on non-convex problems.  Our main contribution can be summarized as follows:
\begin{itemize}
	\item We take an in-depth investigation into the key component in SAM's adversarial perturbation and identify that the  stochastic gradient noise plays the most significant role. 
	\item We propose a novel variant of SAM called F-SAM by eliminating the undesired full gradient component and harnessing the beneficial stochastic gradient noise for adversarial perturbation to enhance generalization. 
	\item Extensive experimental results demonstrate that F-SAM improves SAM's generalization, aids in training, and exhibits better robustness across different perturbation radii.
\end{itemize}

\begin{figure*}[!t]
	\centering
	\begin{subfigure}{0.335\linewidth}
		\centering
		\includegraphics[width=1\linewidth]{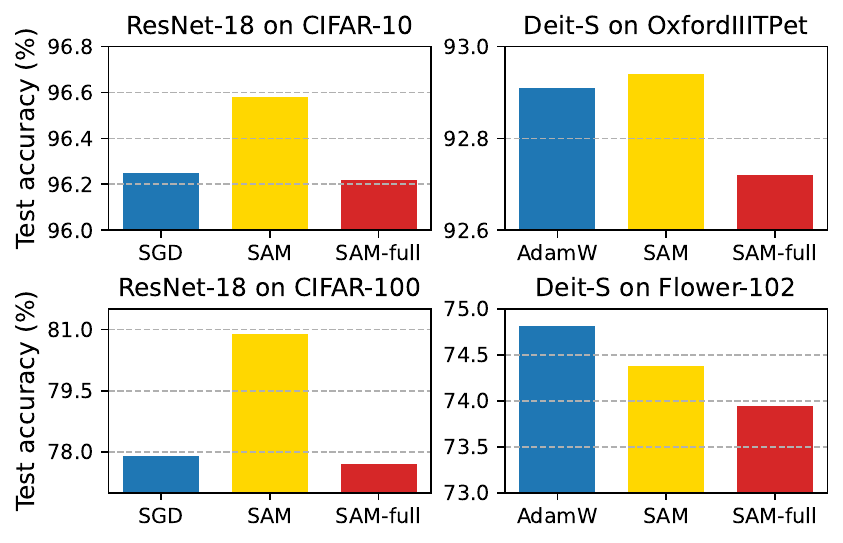}
		\caption{}
		\label{fig:sam-fb}
	\end{subfigure}
	\hspace{0.02in}
	\begin{subfigure}{0.345\linewidth}
		\centering
		\includegraphics[width=1\linewidth]{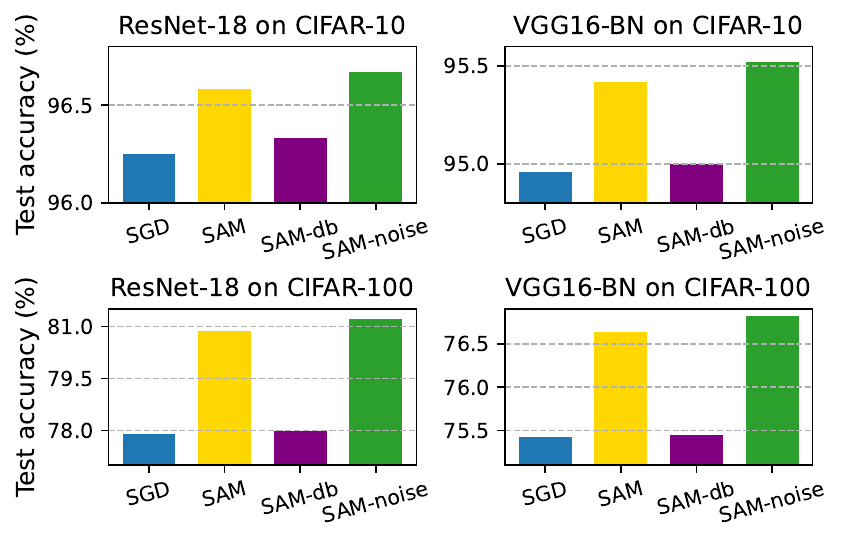}
		\caption{}  
		\label{fig:sam-db}
	\end{subfigure}
	\hspace{0.02in}
	\begin{subfigure}{0.26\linewidth}
		\centering
		\hspace{0.02in}
		\includegraphics[width=1\linewidth]{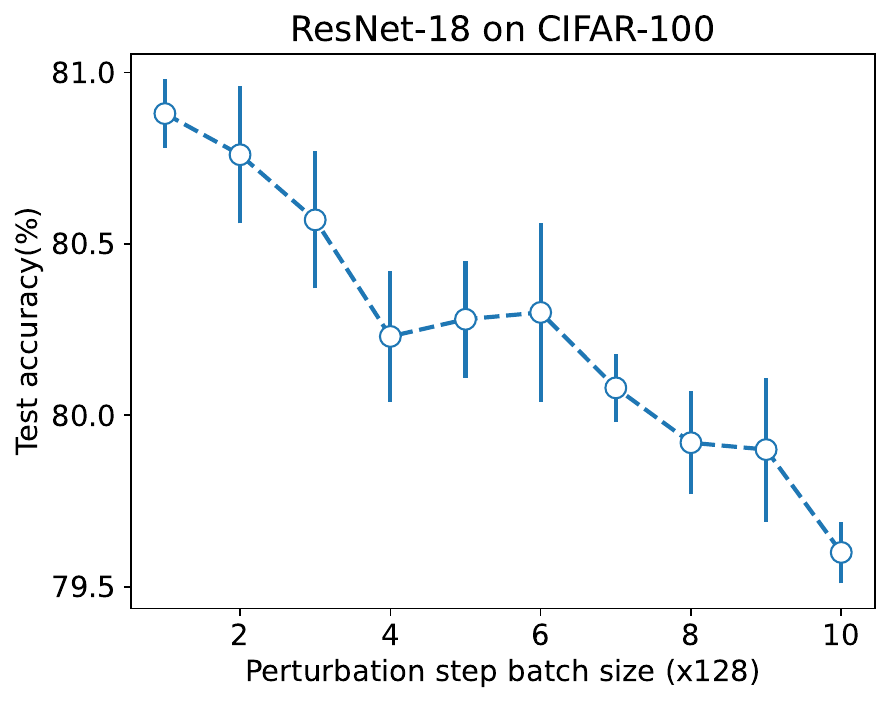}
		\caption{}  
		\label{fig:sam-strength}
	\end{subfigure}
	\caption{Performance comparison of different versions of SAM with SGD/AdamW as its base optimizer.  In (a),  SAM-full denotes the version of SAM using the full gradient component as the perturbation. In (b), SAM-db represents the SAM using different random-selected data batch for perturbation and its following minimization step.  (c) compares  SAM using different minibatch size to compute the perturbation but always fixing minibatch size of 128 for the following minimization step. 
} 
	\label{fig:investigation}
\end{figure*}

\section{Related Work}

\textbf{Sharp Minima and Generalization.}
The connection between the flatness of local minima and generalization has received a rich body of studies \cite{hochreiter1997flat, chaudhari2017entropy, keskar2017large,dinh2017sharp, izmailov2018averaging, li2018visualizing,zhou2020towards}.
Recently, many works have tried to improve the model generalization by seeking flat minima \cite{chaudhari2017entropy, wen2018smoothout, tsuzuku2020normalized, foret2020sharpness, zheng2021regularizing, bisla2022low, li2022efficient,li2023revisiting}. For example, Chaudhari et al. \cite{chaudhari2017entropy} proposed Entropy-SGD that actively searches for flat regions by minimizing local entropy.
Wen et al. \cite{wen2018smoothout} designed SmoothOut framework to smooth out the sharp minima and obtain generalization improvements.
Notably, sharpness-aware minimization (SAM) \cite{foret2020sharpness} provides a generic training scheme for seeking flat minima by formulating the optimization as a min-max problem and encourage parameters sitting in neighborhoods with uniformly low loss, achieving state-of-the-art generalization improvements across various tasks. 
Later, a line of works improves  SAM from the perspective of the neighborhood's geometric measure \cite{kwon2021asam, kim2022fisher}, surrogate loss function \cite{zhuang2022surrogate}, friendly adversary \cite{li2023enhancing}, and training efficiency \cite{du2022efficient, liu2022towards, du2022sharpness,zhao2022ss,jiangadaptive}. 


\vspace{2mm}
\noindent
\textbf{Understanding SAM.}
Despite the empirical success of SAM across various tasks, a deep understanding of its generalization performance is still limited.
Foret et al. \cite{foret2020sharpness} explained the success of SAM via using a PAC-Bayes generalization bound. However, they did not analyze the key components of SAM that contributes to its generalization improvement.  
Andriushchenko et al. \cite{andriushchenko2022towards} investigated the implicit bias of SAM for diagonal linear networks. 
Wen et al. {\cite{wen2022sharpness}} demonstrated that SAM minimizes the top eigenvalue of Hessian in the full-batch setting and thus improves the flatness of the minima.
Chen et al. \cite{chen2023does} studied   SAM on the non-smooth convolutional ReLU networks, and explained  its  success because of  its ability to prevent noise learning. While previous works primarily focus on simplified networks or ideal loss objectives, we aim to deepen the understanding of SAM by undertaking a deep investigation into the key components that contribute to its practical generalization improvement.


\section{Preliminaries}
\label{sec:preliminary} 

\noindent\textbf{SAM.}  Let $f(\boldsymbol{x};\boldsymbol{w})$ parameterized by  $\boldsymbol{w} \in \mathbb{R}^d$ be a neural network, and $L(f(\boldsymbol{x}_i;\boldsymbol{w}), \boldsymbol{y}_i)$ ($L_i(\boldsymbol{w})$ for short)   denote the loss to measure the discrepancy between the  prediction $f(\boldsymbol{x}_i;\boldsymbol{w})$ and the ground-truth label $\boldsymbol{y}_i$.  Given a dataset $\mathcal{S}={ (\boldsymbol{x}_i,\boldsymbol{y}_i) }_{i=1}^n$ i.i.d. drawn from a data distribution $\mathcal{D}$, the empirical training loss is often defined as 
\begin{equation}
   L(\boldsymbol{w})= \frac{1}{n}\sum\nolimits_{i=1}^nL_i(\boldsymbol{w}). 
\end{equation}
To solve this problem, one often uses SGD  or Adam to optimize it.  To further improve the generalization performance, 
SAM~\cite{foret2020sharpness} aims to minimize the worst-case loss within a defined neighborhood  for guiding the training process towards flat minima. The objective function of SAM is given by:
\begin{equation}
	L^{\rm SAM}(\boldsymbol{w})= \max\nolimits_{\| \boldsymbol{\epsilon} \|_2 \le \rho} L(\boldsymbol{w}+\boldsymbol{\epsilon}),
	\label{equ:sam}
\end{equation}
where $\rho$ denotes the neighborhood radius. Eqn.~\eqref{equ:sam} shows that  SAM seeks to minimize the maximum loss over the neighborhood surrounding the current weight  $\boldsymbol{w}$. By doing so, SAM encourages the optimization process to converge to flat minima which often enjoys better generalization  performance~\cite{hochreiter1997flat,dinh2017sharp, li2018visualizing,zhou2020towards}.

To efficiently optimize $	L^{\rm SAM}(\boldsymbol{w})$, SAM first approximates  Eqn.~\eqref{equ:sam} via first-order expansion and compute the adversarial perturbation $\boldsymbol{\epsilon}_s$ as follows:
\begin{equation}
	\boldsymbol{\epsilon}_s\approx \mathop{ \arg \max}\nolimits_{\| \boldsymbol{\epsilon} \|_2\le \rho} \boldsymbol{\epsilon}^{\top} \nabla_{\boldsymbol{w}} L(\boldsymbol{w})= \rho \cdot\norm(\nabla_ {\boldsymbol{w}} L(\boldsymbol{w})). 
	\label{equ:sam-first-order}
\end{equation}
where $\norm(\boldsymbol{x})= \boldsymbol{x}/ \|\boldsymbol{x}\|_2$. 
Subsequently, one can compute the gradient at the perturbed point $\boldsymbol{w}+\boldsymbol{\epsilon}_s$, and then use the updating step of the base SGD optimizer to update  
\begin{equation}\label{eqn:gradient}
	\boldsymbol{w}_{t+1} = \boldsymbol{w}_t - \gamma	\nabla L_{\mathcal{B}} (\boldsymbol{w}) |_{\boldsymbol{w}_t+\boldsymbol{\epsilon}_t}. 
\end{equation}
Other base optimizers, e.g., Adam, can also be used to update the model parameters in Eqn.~\eqref{eqn:gradient}. 
%

\noindent\textbf{Assumptions.}  Before delving into our analysis, we first make some standard assumptions in stochastic optimization \cite{ghadimi2013stochastic,karimi2016linear,andriushchenko2022towards,jiangadaptive} that will be used in our theoretical analysis.
\begin{assumption}[$\beta$-smoothness]
	\label{assumption1}
	Assume the loss function $L:\mathbb{R}^d \mapsto \mathbb{R}$ to be $\beta$-smooth. There exists $\beta>0$ such that
	\begin{align}
		\| \nabla L(\boldsymbol{w})-\nabla L(\boldsymbol{v}) \|_2 \le \beta \| \boldsymbol{w}-\boldsymbol{v} \|_2, \ \ \forall \boldsymbol{w}, \boldsymbol{v} \in  \mathbb{R}^d.
	\end{align}
\end{assumption}

\begin{assumption}[Bounded variance]
	\label{assumption2}
	There exists a constant $M>0$ for any data batch $\mathcal{B}$ such that
	\begin{align}
\mathbb{E}	\left[\| \nabla L_\mathcal{B}(\boldsymbol{w}) - \nabla L(\boldsymbol{w})  \|_2^2 \right] \le M, \quad \forall \boldsymbol{w} \in  \mathbb{R}^d.
	\end{align}
\end{assumption}

\begin{assumption}[Bounded gradient]
	\label{assumption3}
	There exists $G>0$ for any data batch $\mathcal{B}$ such that
	\begin{align}
	 \mathbb{E}	\left[	\| \nabla L_\mathcal{B}(\boldsymbol{w}) \|_2  \right] \le G, \quad \forall \boldsymbol{w} \in  \mathbb{R}^d.
	\end{align}
\end{assumption}

\section{Empirical Analysis of SAM}
\label{sec:investigation}
Here we conduct a set of experiments to identify the effective components in the adversarial perturbation that contributes to SAM's improved generalization performance. 
To this end, we decompose SAM's perturbation direction  $\nabla L_{\mathcal{B}}(\boldsymbol{w})$  in Eqn.~\eqref{equ:sam-first-order}  into two orthogonal components. The first one is the full gradient component $\mathrm{Proj}_{\nabla L(\boldsymbol{w})} \nabla L_{\mathcal{B}}(\boldsymbol{w}) $ which  is the projection  onto the full gradient direction, i.e., 
\begin{equation*}
\mathrm{Proj}_{\nabla L(\boldsymbol{w})} \nabla L_{\mathcal{B}}(\boldsymbol{w}) = \cos(\nabla L(\boldsymbol{w}), L_{\mathcal{B}}(\boldsymbol{w})) \nabla L(\boldsymbol{w}),
\end{equation*}
where $\cos(\cdot, \cdot)$ denotes the cosine similarity function.  
Another one is  the residual projection, i.e., the batch-specific  stochastic gradient noise component which is defined as    
\begin{equation}\label{equ:residual-projection}
\mathrm{Proj}^\top_{\nabla L(\boldsymbol{w})} \nabla L_{\mathcal{B}}(\boldsymbol{w}) =  \nabla L_{\mathcal{B}}(\boldsymbol{w})  - \mathrm{Proj}_{\nabla L(\boldsymbol{w})} \nabla L_{\mathcal{B}}(\boldsymbol{w}).
\end{equation}
In the following, we will investigate the effects of these two orthogonal components to the performance of SAM.

\subsection{Effect of Full Gradient Component} 
Here we investigate the impact of the full gradient component
on the generalization performance of SAM. 
Specifically, for each training iteration, we set the perturbation $\boldsymbol{\epsilon}_s$  as $\boldsymbol{\epsilon}_s=\rho \cdot \norm(\mathrm{Proj}_{\nabla L(\boldsymbol{w})} \nabla L_{\mathcal{B}}(\boldsymbol{w}) )$ where $\norm(\boldsymbol{x})= \boldsymbol{x}/ \|\boldsymbol{x}\|_2$, and then follow the minimization step in vanilla SAM to minimize the perturbed training loss. For brevity, this modified SAM version is called ``SAM-full" since it uses the full gradient direction as the perturbation direction. 
Then we follow the standard training settings of SAM and compare the performance of SGD/AdamW, SAM, and SAM-full in training from scratch and transfer learning tasks.  See more training details in Appendix~A5.

 \cref{fig:sam-fb} (a) summarizes the results. One can observe that SAM significantly outperforms the base optimizer in most scenarios, except for the Flower-102 dataset~\cite{nilsback2008automated}, where the training data are very limited (see  \cref{sec:transfer} for more discussion). However, as shown in  \cref{fig:sam-fb} (a),  SAM-full indeed impairs the performance of base optimizers, such as SGD base optimizer on the CIFAR-10, CIFAR-100~\cite{krizhevsky2009learning}, and OxFordIIITPet datasets~\cite{parkhi2012cats}, and even severely hurt the performance of  AdamW base optimizer on the Flower-102 dataset, of which the training sample number is small. This implies that 1) utilizing the full gradient  
component $\mathrm{Proj}_{\nabla L(\boldsymbol{w})} \nabla L_{\mathcal{B}}(\boldsymbol{w})$  
as the perturbation direction actually does not improve generalization and can even have negative effects; 2) the effectiveness of SAM primarily arises from the presence of stochastic gradient noise component  $\mathrm{Proj}^\top_{\nabla L(\boldsymbol{w})} \nabla L_{\mathcal{B}}(\boldsymbol{w})$ within the minibatch gradient $\nabla L_{\mathcal{B}}(\boldsymbol{w})$ in adversarial perturbation.

To provide more evidence, we further modify the perturbation step in SAM by increasing the minibatch size in Eqn.~\eqref{equ:sam-first-order} to compute the adversarial perturbation . As illustrated in  \cref{fig:sam-strength}, we use a minibatch $\mathcal{B}'$ with size $k\times 128$ to compute the perturbation in SAM by following Eqn.~\eqref{equ:sam-first-order}, where $k$ is selected from $\{1, 2, 3, \cdots, 10\}$. However, we maintain the use of a minibatch $\mathcal{B}$ of size 128 for the second minimization step of standard SAM  in Eqn.~\eqref{eqn:gradient}. Here we ensure the sample set $\mathcal{B}$ in the second minimization belong to the sample set $\mathcal{B}'$ in the first perturbation step, namely, $\mathcal{B} \subseteq \mathcal{B}'$, which allows to better observe the effects of large minibatch to the performance of SAM.   
For the perturbation $\boldsymbol{\epsilon}_s = \rho \cdot \norm{\nabla L_{\mathcal{B}'}(\boldsymbol{w})}$, its direction gradually aligns more closely with 
the full gradient direction along with increment of $k$. 
Consequently, as the minibatch size grows, the contribution of the full gradient component in the perturbation step becomes stronger, while the residual stochastic gradient noise component weakens. This analysis helps to reveal the perturbation effects of the full gradient direction to  SAM's performance.

By observing \cref{fig:sam-strength},  one can find that when minibatch size of $\mathcal{B}'$ in the first adversarial perturbation step increases, the classification accuracy of SAM gradually decreases when using ResNet18 on CIFAR-100 dataset. See more similar experimental results in Appendix~A5.
All these results together   indicate  that  strengthening the full gradient component in the perturbation step indeed  impairs the generalization performance of vanilla SAM.

\subsection{Effect of Stochastic Gradient Noise Component}

To delve into the role of the batch-specific stochastic gradient noise in the generalization improvement of SAM, we first replace the stochastic gradient noise component of the current minibatch data with another random minibatch. That is, for each training iteration, we use two different minibatch data of the same size to do the first adversarial perturbation step~\eqref{equ:sam-first-order} and the second minimization step~\eqref{eqn:gradient}. We term this modified version ``SAM-db''. As illustrated in \cref{fig:sam-db},  we observe that for the four training settings, SAM-db with SGD as its base optimizer often does not make significant improvements in generalization compared to  SGD. 
Note that the adversarial perturbation in SAM-db contains a similar full gradient component as vanilla SAM (in terms of expectation), but the key difference lies in replacing the stochastic gradient noise component from the current minibatch data with another random minibatch. However, this substitution ultimately hinders generalization.
Therefore, we conclude that the stochastic gradient noise associated with the minibatch in decent step in perturbation plays a pivotal role in improving the generalization of SAM. 

To further solidify this observation, 
we modify SAM by setting the adversarial perturbation $\boldsymbol{\epsilon}_s$ in vanilla SAM as $\boldsymbol{\epsilon}_s = \rho\cdot \norm(\mathrm{Proj}^\top_{\nabla L(\boldsymbol{w})} \nabla L_{\mathcal{B}}(\boldsymbol{w}))$. Note, in this modification, we compute the full gradient for each iteration and then follow Eqn.~\eqref{equ:residual-projection} to calculate the stochastic gradient noise $\mathrm{Proj}^\top_{\nabla L(\boldsymbol{w})}  \nabla L_{\mathcal{B}}(\boldsymbol{w})$.    
	We refer to this modified version as ``SAM-noise". 
As illustrated in \cref{fig:sam-db}, one can observe that SAM-noise not only restores the generalization performance of SAM but even exhibits a notable enhancement. 
This result highlights that exclusively utilizing the stochastic gradient noise component as  SAM's perturbation can further enhance its generalization ability, thereby inspiring our  F-SAM algorithm in the next section.  

\section{The F-SAM Algorithm}
In this section, we will present our proposed F-SAM algorithm. We begin by introducing the efficient estimation of the stochastic gradient noise component. We then describe the algorithmic steps of our proposed F-SAM. Finally, we formulate the loss objective of our F-SAM algorithm and highlight the distinctions from vanilla SAM.


\subsection{Estimation of Stochastic Gradient Noise Component}
From the empirical results in \cref{sec:investigation}, we know that the full gradient component $\mathrm{Proj}_{\nabla L(\boldsymbol{w})} \nabla L_{\mathcal{B}}(\boldsymbol{w})$ in the perturbation   impairs the performance of SAM, while the batch-specific stochastic gradient noise component  $\mathrm{Proj}^\top_{\nabla L(\boldsymbol{w})} \nabla L_{\mathcal{B}}(\boldsymbol{w})$ is essential to improve performance.  Accordingly, it is natural to remove the full gradient component and only use the  stochastic gradient noise component. 

To this end, we first reformulate the batch-specific stochastic gradient noise component in Eqn.~\eqref{equ:residual-projection} as 
\begin{equation}\label{equ:stochastic-noise}
	\mathrm{Proj}^\top_{\nabla L(\boldsymbol{w})} \nabla L_{\mathcal{B}}(\boldsymbol{w}) \!=\!  \nabla L_{\mathcal{B}}(\boldsymbol{w})  - \sigma \nabla L(\boldsymbol{w}).
\end{equation}
where $\sigma = \cos(\nabla L(\boldsymbol{w}), \nabla L_{\mathcal{B}}(\boldsymbol{w})).$ From Eqn.~\eqref{equ:stochastic-noise}, one can observe that to compute  $	\mathrm{Proj}^\top_{\nabla L(\boldsymbol{w})} \nabla L_{\mathcal{B}}(\boldsymbol{w})$, we need to compute the  minibatch gradient $\nabla L_{\mathcal{B}}(\boldsymbol{w})$ and also the full gradient $\nabla L(\boldsymbol{w})$ at each training  iteration. For $\nabla L_{\mathcal{B}}(\boldsymbol{w})$, it is also computed in vanilla SAM and thus does not bring extra computation overhead. Regarding the full gradient $\nabla L(\boldsymbol{w})$, it is computed on the whole dataset, and thus is  computationally prohibitive in practice.

To address this issue, we resort to estimating $\nabla L(\boldsymbol{w})$ with an exponentially moving average (EMA)  which accumulates the  historical minibatch gradients as follows:
\begin{align}
	\boldsymbol{m}_t = \lambda \boldsymbol{m}_{t-1} +  (1 - \lambda) \nabla L_{\mathcal{B}_t}(\boldsymbol{w}_t),
\end{align}
where $\lambda>0$  is a hyper-parameter.  This approach enables us to approximate the full gradient $\nabla L(\boldsymbol{w}_t)$ with minimal additional computational overhead. We also prove that $	\boldsymbol{m}_t $ is a good estimation to the full gradient  $\nabla L(\boldsymbol{w}_t)$ as shown in following Theorem~\ref{theoremestimation} with its proof in Appendix~A1.




\begin{theorem}\label{theoremestimation}
Suppose Assumption \ref{assumption1}, \ref{assumption2}, and \ref{assumption3} hold. Assume that SAM uses SGD as the base optimizer with a learning rate $\gamma$ to update the model parameter in Eqn.~\eqref{eqn:gradient}.  Then by setting $\lambda=1-C\gamma^{2/3}$, after $T > C^\prime \gamma^{-2/3}$  training iterations, with probability $1-\delta$, we have  
\begin{equation*}
    \Phi_T = \|\boldsymbol{m}_T - \nabla L(\boldsymbol w_T)\|_2  \leq  \mathcal{O}\big(\gamma^{\frac{1}{3}} \beta^{\frac{1}{3}} G^{\frac{1}{3}} M^{\frac{1}{3}}  \log\big(\frac{1}{\delta} \big) \big),
\end{equation*}
where $C$ and $C'$ are two universal constants.


\end{theorem}
Theorem~\ref{theoremestimation} shows that  
the error bound  $ \Phi_t$ between   the full gradient $\nabla L(\boldsymbol{w}_t)$ and its EMA estimation $\boldsymbol{m}_t $ is at the order of $\mathcal{O}(\gamma^{1/3})$. {On the non-convex problem, learning rate  $\gamma$ is often set as $\mathcal{O}(1/\sqrt{T})$ to ensure convergence as shown in our Theorem~\ref{thm:convergence} in Section~\ref{convergence}. } This implies that the error bound  $ \Phi_t$ is $\mathcal{O}(T^{-1/6})$ and thus is small since the training iteration $T$ is often large.  So $	\boldsymbol{m}_t $ is a good estimation the full gradient  $\nabla L(\boldsymbol{w}_t)$.

In this way, the  stochastic gradient noise component in Eqn.~\eqref{equ:residual-projection} is approximated as  
 \begin{equation}\label{adasfwaasde}
 	\begin{split}
 	\mathrm{Proj}^\top_{\nabla L(\boldsymbol{w}_t)} \nabla L_{\mathcal{B}}(\boldsymbol{w}_t)\approx   & 	\widetilde{\mathrm{Proj}}^\top_{\nabla L(\boldsymbol{w}_t)} \nabla L_{\mathcal{B}}(\boldsymbol{w}_t), \\
 	\end{split}
 \end{equation}
 where $\widetilde{\mathrm{Proj}}^\top_{\nabla L(\boldsymbol{w}_t)} \nabla L_{\mathcal{B}}(\boldsymbol{w}_t) = \nabla L_{\mathcal{B}}(\boldsymbol{w}_t)  - \sigma_t \boldsymbol{m}_t$ and  $\sigma_t = \cos(\boldsymbol{m}_t, \nabla L_{\mathcal{B}}(\boldsymbol{w}_t)).$  
In the real network  training,  the training frameworks, e.g., PyTorch, often  uses propagation to update $\boldsymbol{m}_t$ and $\nabla L_{\mathcal{B}}(\boldsymbol{w}_t)$ from layer to layer. So one needs to wait until the training framework finishes the updating of all layers, and then compute $\sigma_t $ which is not efficient, especially for modern over-parameterized networks. Moreover,  both $\boldsymbol{m}_t$ and $\nabla L_{\mathcal{B}}(\boldsymbol{w}_t)$ are high dimensional,  their  cosine computation can be unstable due to the possible noises in the approximation $\boldsymbol{m}_t$. 
So to improve the training efficiency and stability,  we set $\sigma_t$ as a constant $\sigma$, and obtain  \begin{equation}\label{adasfwaasadade}
 	\begin{split}
\widetilde{\mathrm{Proj}}^\top_{\nabla L(\boldsymbol{w}_t)} \nabla L_{\mathcal{B}}(\boldsymbol{w}_t)  
=   \nabla L_{\mathcal{B}}(\boldsymbol{w}_t)  - \sigma \boldsymbol{m}_t.
 	\end{split}
 \end{equation}
 The experimental results in Section \ref{sec:experiments} show that this practical setting achieves satisfactory performance.  In theory, we also prove that this practical setting is validated as shown in the following Lemma~\ref{theorem2} with its proof in Appendix~A2.
  \begin{lemma}\label{theorem2}
 	Assuming $\boldsymbol{m}_t$ is a unbiased estimator to the full gradient $\nabla L(\boldsymbol{w}_t)$. With $\sigma=1$, we have
 	\begin{align}
 		\begin{split}
 			\mathbb{E} \left \langle \widetilde{\mathrm{Proj}}^\top_{\nabla L(\boldsymbol{w}_t)} \nabla L_{\mathcal{B}}(\boldsymbol{w}_t), \nabla L(\boldsymbol{w}_t) \right \rangle=0.
 		\end{split}
 	\end{align}
 \end{lemma}
Lemma~\ref{theorem2} shows that $\widetilde{\mathrm{Proj}}^\top_{\nabla L(\boldsymbol{w}_t)} \nabla L_{\mathcal{B}}(\boldsymbol{w}_t)$ is orthogonal to the full gradient $\nabla L(\boldsymbol{w}_t)$ in expectation when  $\sigma=1$. So using $\widetilde{\mathrm{Proj}}^\top_{\nabla L(\boldsymbol{w}_t)} \nabla L_{\mathcal{B}}(\boldsymbol{w}_t)$  as the perturbation direction can remove the effects of the full gradient $\nabla L(\boldsymbol{w}_t)$. As Theorem~\ref{theoremestimation} shows that  $\boldsymbol{m}_t$ is good estimation to    full gradient $\nabla L(\boldsymbol{w}_t)$,   the assumption in  Lemma~\ref{theorem2} is not restrictive.

\subsection{Algorithmic Steps of F-SAM}
After estimating the stochastic gradient noise component in Eqn.~\eqref{adasfwaasadade}, one can use it to replace the adversarial perturbation direction in vanilla SAM, and straightforwardly develop our proposed F-SAM.  Specifically, for the $t$-th training iteration, following the algorithmic steps in vanilla SAM, we first define the perturbation $\boldsymbol{\epsilon}$ as follows:
\begin{align}
\boldsymbol{\epsilon}_t =  \rho \cdot\norm \big(\widetilde{\mathrm{Proj}}^\top_{\nabla L(\boldsymbol{w}_t)} \nabla L_{\mathcal{B}}(\boldsymbol{w}_t)\big),
	\label{equ:dts}
\end{align}
where $\rho$ is the  radius.  
Then, following SAM, we use SGD as its base optimizer, and updates the model parameter $\boldsymbol{w}_t$ by using Eqn.~\eqref{eqn:gradient} in Section~\ref{sec:preliminary}.  We summarize the algorithmic steps of F-SAM with SGD as base optimizer  in  Algorithm
\ref{algorithm:fsam}. One can also use other base optimizers, e.g., AdamW, to update the model parameters in Eqn.~\eqref{eqn:gradient}. 

\begin{algorithm}[t]
	\caption{F-SAM algorithm}
	\label{algorithm:fsam}
	\KwIn{Loss function $L(\boldsymbol{w})$,  training datasets $\mathcal{S}=\{ (\boldsymbol{x}_i,\boldsymbol{y}_i) \}_{i=1}^n$, minibatch size $b$, neighborhood size $\rho$, learning rate $\gamma$, momentum factor $\lambda$, projection constant $\sigma$} 
	\KwOut{Trained weight $\boldsymbol{w}$}  
	\BlankLine
	Initialize $\boldsymbol{w}_0$, $t\gets 0$, $\boldsymbol{m}_{-1} = \boldsymbol{0}$;\\ 
	\While{not converged}{
		Sample a minibatch data $\mathcal{B}$ of size $b$ from $\mathcal{S}$; \\
		$\boldsymbol{g}_t=\nabla L_{\mathcal{B}}(\boldsymbol{w}_{t})$;\\
		$\boldsymbol{m}_{t}=\lambda \boldsymbol{m}_{t-1}+(1-\lambda)\boldsymbol{g}_t$;\\
		Compute adversarial adversarial perturbation: \\ \quad \quad $  \boldsymbol{\epsilon}_t = \rho \frac{\boldsymbol{\mathrm{d}}_t}{\|\boldsymbol{\mathrm{d}}_t\|}$ where $\boldsymbol{\mathrm{d}}_t=\boldsymbol{g}_t-\sigma \boldsymbol{m}_{t}$; \\
		Compute the gradient approximation  $\boldsymbol{g}$:\\
		\quad \quad $\boldsymbol{g}_t' \gets \nabla L_{\mathcal{B}}(\boldsymbol{w}_{t} + \boldsymbol{\epsilon}_t)$;
		\\
		Update $\boldsymbol{w}$ using gradient descent:\\
		\quad \quad $\boldsymbol{w}_{t+1}\gets \boldsymbol{w}_{t}-\gamma \boldsymbol{g}_t'$; \\
		$t \gets t + 1$; \\
	}
	\Return $\boldsymbol{w}_t$.
\end{algorithm}

\subsection{A Friendly Perspective of the Loss Objective}
Here we provide an intuitive understanding on the difference between SAM and F-SAM. For a minibatch data $\mathcal{B}$, SAM solves the  maximization problem to seek adversarial perturbation: $\boldsymbol{\epsilon}_s^{\rm SAM}= \mathop{ \arg \max}_{\| \boldsymbol{\epsilon} \|_2 \le \rho} L_{\mathcal{B}}(\boldsymbol{w}+\boldsymbol{\epsilon}).$
%
In contrast,   by observing $\widetilde{\mathrm{Proj}}^\top_{\nabla L(\boldsymbol{w}_t)} \nabla L_{\mathcal{B}}(\boldsymbol{w}_t)  $   in Eqn.~\eqref{adasfwaasadade}, F-SAM indeed computes the adversarial perturbation by maximizing the current minibatch loss while  minimizing the loss on the entire dataset:
\begin{align*}
	{\boldsymbol{\epsilon}}^{\rm F \raisebox{0mm}{-}SAM}_s=\mathop{ \arg \max}\limits_{\| \boldsymbol{\epsilon} \|_2 \le \rho} L_{\mathcal{B}} (\boldsymbol{w}+\boldsymbol{\epsilon})- \sigma L_{\mathcal{D}} (\boldsymbol{w}+\boldsymbol{\epsilon}).
\end{align*}
Then combining the following minimization problem, one can write F-SAM's bi-level optimization problem as
\begin{align*}
\begin{split}
	\min_{\boldsymbol{w}} &~~~\mathbb{E}_{\mathcal{B}} \left [ L_{\mathcal{B}}(\boldsymbol{w}+\boldsymbol{\epsilon}_s)\right ] \\
	\mathrm{s.t.}  &~~~{\boldsymbol{\epsilon}}_s=\mathop{ \arg \max}\nolimits_{\| \boldsymbol{\epsilon} \|_2 \le \rho} L_{\mathcal{B}} (\boldsymbol{w}+\boldsymbol{\epsilon})-\sigma L_{\mathcal{D}} (\boldsymbol{w}+\boldsymbol{\epsilon}).
\end{split}
\end{align*}
The insight behind is that F-SAM aims to find an adversarial perturbation ${\boldsymbol{\epsilon}}_s$ that increases the loss sharpness of the current minibatch data while minimizing the impact on the loss sharpness of the other data points as much as possible. Then F-SAM minimizes the loss value and sharpness of current minibatch data. So the adversarial perturbation in F-SAM is  ``friendly'' to other data points, yielding consistent sharpness minimization.  Moreover, this friendliness enables high robustness of F-SAM  to the perturbation radius as shown in \cref{sec:ablation}.   But for vanilla SAM, its adversarial perturbation increases the loss sharpness of current minibatch data and other data points but minimizes the loss value and sharpness of only current minibatch data. So in vanilla SAM, the second minimization step indeed does not consider the loss sharpness increment caused by its first adversarial perturbation step. This inconsistency may impair the sharpness optimization and thus limit the generalization.

\subsection{Convergence Analysis}\label{convergence}
In this section, we analyze the convergence properties of the F-SAM algorithm under non-convex setting. 
We follow basic assumptions that are standard in convergence analysis of stochastic optimization \cite{ghadimi2013stochastic,karimi2016linear,andriushchenko2022towards}.
\begin{theorem}
	\label{thm:convergence}
Assume Assumption \ref{assumption1} and \ref{assumption2} hold.  Assume that SAM uses SGD as the base optimizer with a learning rate $\gamma$ to update the model parameter in Eqn.~\eqref{eqn:gradient}.   By setting learning rate $\gamma=\frac{\gamma_0}{\sqrt{T}} \le {1}/{\beta}$ and the perturbation radius  $\rho_t=\frac{\rho_0}{\sqrt{t}}$, we have 
	\begin{align}
		\begin{split}
\frac{1}{{T}} \sum_{t=1}^{T}  \mathbb{E}   \|\nabla L(\boldsymbol{w}_{t}) \|^2\le
\frac{2 \Delta}{\gamma_0 \sqrt{T}}+\frac{\Theta}{\sqrt{T}}+\frac{\Pi \log T}{\sqrt{T}},
		\end{split}
	\end{align}
where $\Delta =  \mathbb{E} [L(\boldsymbol{w}_{0}) -L(\boldsymbol{w}^{*})  ]$	with the optimal solution $\boldsymbol{w}^*$  to $L(\boldsymbol{w})$, $\Theta=2\beta M \gamma_0+ \rho_0^2 \beta^3\gamma_0$,  and $\Pi = \rho_0^2 \beta^2$.
\end{theorem}
\noindent
See the proof   in Appendix~A3.
For non-convex stochastic optimization, Theorem \ref{thm:convergence} shows that F-SAM has the convergence rate   $\mathcal{O}(\log T/\sqrt{T})$ and share the same convergence speed as  SAM. But F-SAM enjoys better generalization performance than SAM as shown in Section~\ref{sec:experiments}. 

\begin{table*}[!t]
	\centering
	\begin{tabular}{c|ccc|ccc}
		\toprule
		CIFAR-10 & SGD   &ASAM &FisherSAM &SAM &F-SAM (ours) \\
		\midrule
		\textbf{VGG16-BN} &94.96$_{\pm 0.15}$   &95.57$_{\pm 0.05}$  &95.55$_{\pm 0.08}$ &95.42$_{\pm 0.05}$ &\textbf{95.62}$_{\pm 0.11}$\\
		\textbf{ResNet-18} &96.25$_{\pm 0.06}$  &96.63$_{\pm 0.15}$ &96.72$_{\pm 0.03}$ &96.58$_{\pm 0.10}$ &\textbf{96.75}$_{\pm 0.09}$\\
		\textbf{WRN-28-10} &97.08$_{\pm 0.16}$  &97.64$_{\pm 0.13}$ &97.46$_{\pm 0.18}$ &97.32$_{\pm 0.11}$ &\textbf{97.53}$_{\pm 0.11}$  \\
		\textbf{PyramidNet-110} &97.63$_{\pm 0.09}$  &97.82$_{\pm 0.07}$ &97.64$_{\pm 0.09}$ &97.70$_{\pm 0.10}$ &\textbf{97.84}$_{\pm 0.05}$ \\
		\bottomrule
	\end{tabular}
 \vspace{-2mm}
	\caption{Test accuracy (\%) comparison of various neural networks on CIFAR-10.}
	\label{tab:cifar10}
 \vspace{-3mm}
\end{table*}

\begin{table*}[!t]
	\centering
	\begin{tabular}{c|ccc|cc}
		\toprule
		CIFAR-100 & SGD &ASAM &FisherSAM  &SAM &F-SAM (ours) \\
		\midrule
		\textbf{VGG16-BN} &75.43$_{\pm 0.19}$  &76.27$_{\pm 0.35}$ &76.90$_{\pm 0.37}$ &76.63$_{\pm 0.20}$ &\textbf{77.08}$_{\pm 0.17}$\\
		\textbf{ResNet-18} &77.90$_{\pm 0.07}$  &{81.68}$_{\pm 0.12}$ &80.99$_{\pm 0.13}$ &80.88$_{\pm 0.10}$ &\textbf{81.29}$_{\pm 0.12}$ \\
		\textbf{WRN-28-10} &81.71$_{\pm 0.13}$  &84.99$_{\pm 0.22}$ &84.91$_{\pm 0.07}$ &84.88$_{\pm 0.10}$ &\textbf{85.16}$_{\pm 0.07}$ \\
		\textbf{PyramidNet-110} &84.65$_{\pm 0.11}$  &86.47$_{\pm 0.09}$ &86.53$_{\pm 0.07}$ &86.55$_{\pm 0.08}$ &\textbf{86.70}$_{\pm 0.14}$ \\
		\bottomrule
	\end{tabular}
 \vspace{-2mm}
	\caption{Test accuracy (\%) comparison of various neural networks on CIFAR-100.}
	\label{tab:cifar100}
 \vspace{-4mm}
\end{table*}

\section{Numerical Experiments}
\label{sec:experiments}
Here we test F-SAM on   various tasks and  network architectures under  two popular settings: 1) training from scratch, and 2)  transfer learning  by fine-tuning pretrained models. Moreover,   ablation study and additional experiments show more insights on F-SAM.

\subsection{Training From Scratch}
\label{sec:trainingfromscratch}


\textbf{~~~CIFAR.}
We first evaluate over CIFAR-10/100 datasets \cite{krizhevsky2009learning}. The network architectures we considered include VGG16-BN \cite{simonyan2014very}, ResNet-18 \cite{he2016deep}, WRN-28-10 \cite{ZagoruykoK16} and PyramidNet-110 \cite{han2017deep}. 
The first three models are trained for 200 epochs while for PyramidNet-110 we train for 300 epochs. 
We set the initial learning rate as 0.05 with a cosine learning rate schedule. The momentum and weight decay are set to 0.9 and 0.0005 for SGD, respectively. SAM and variants adopt the same setting except that the weight decay is set to 0.001 following \cite{mi2022make,li2023enhancing}.
We apply standard random horizontal flipping, cropping, normalization, and cutout augmentation \cite{devries2017improved}. For SAM, we set the perturbation radius $\rho$ as 0.1 and 0.2 for CIFAR-10 and CIFAR-100 \cite{mi2022make,li2023enhancing}. For ASAM, we adopt the recommend $\rho$ as 2 and 4 for CIFAR-10 and CIFAR-100 \cite{kwon2021asam}. For FisherSAM, we adopt the recommended $\rho=0.1$. For F-SAM, we use the same $\rho$ as SAM, and tune $\lambda=\{ 0.6, 0.9, 0.95 \}$ although all parameters outperforms SAM. We find that $\lambda=0.9$ works best for WRN-28-10 and PyramidNet-110, while $\lambda=0.6$ achieves the best for others. We simply set $\sigma=1$ for all the following experiments. Experiments are repeated over 3 independent trials.

\begin{table}[htbp]
    \centering
    \small
    \vspace{-3mm}
    \begin{tabular}{c|cc}
    \toprule
         CIFAR-10 &ASAM &F-ASAM (ours) \\
         \midrule
         \textbf{ResNet-18} &96.63$_{\pm 0.15}$ &\textbf{96.77}$_{\pm 0.11}$\\
         \textbf{WRN-28-10} &97.64$_{\pm 0.13}$ &\textbf{97.73}$_{\pm 0.04}$\\
        \midrule
         CIFAR-100 &ASAM &F-ASAM  (ours)  \\
         \midrule
         \textbf{ResNet-18} &81.68$_{\pm 0.12}$ &\textbf{81.82}$_{\pm 0.14}$ \\
         \textbf{WRN-28-10} &84.99$_{\pm 0.22}$ &\textbf{85.24}$_{\pm 0.16}$\\
    \bottomrule
    \end{tabular}
 \vspace{-3mm}
    \caption{Results on integration with ASAM.}
    \label{tab:adaptive}
 \vspace{-3mm}
\end{table}


In \cref{tab:cifar10} and \ref{tab:cifar100}, we observe that F-SAM achieves 0.1 to 0.2 accuracy gain on CIFAR-10 and 0.2 to 0.4 on CIFAR-100 in all tested scenarios. Moreover, F-SAM outperforms SAM's adaptive variants in most cases, and in \cref{tab:adaptive}, we integrate our friendly approach to ASAM (see Appendix~A4) and again observe consistent improvement.

\begin{table}[!b]
	\centering
	\small
	\setlength{\tabcolsep}{3.5pt}
    \vspace{-3mm}
	\begin{tabular}{c|cccc}
		\toprule
		ImageNet &SGD &ASAM &SAM  &F-SAM \\
		\midrule
		ResNet-50 &76.62$_{\pm 0.12}$ &77.10$_{\pm 0.14}$ &77.14$_{\pm 0.16}$  &\textbf{77.35}$_{\pm 0.12}$\\
		\bottomrule
	\end{tabular}
\vspace{-3mm}
	\caption{Test accuracy (\%) comparison on ImageNet.}
\vspace{-3mm}
	\label{tab:imagenet}
\end{table}

\textbf{ImageNet.}
Next, we investigate the performance of F-SAM on larger scale datasets by training ResNet-50 \cite{he2016deep} on ImageNet \cite{deng2009imagenet} from scratch. We set the training epochs to 90 with batch size 128, weight decay 0.0001, and momentum 0.9. We set the initial learning rate as 0.05 with a cosine schedule \cite{li2023enhancing}. We apply basic data preprocessing and augmentation \cite{paszke2017automatic}. 
We set $\rho=0.075$ for all three SAM and variants and $\lambda=0.95$ for F-SAM.
In \cref{tab:imagenet}, we observe that F-SAM outperforms SAM and ASAM and offers an accuracy gain of 0.21 over SAM. This validates the effectiveness of F-SAM on large-scale problems.

\subsection{Transfer Learning}
\label{sec:transfer}


One of the fascinating training pipelines for modern DNNs is transfer learning, i.e., first training a model on a large datasets and then easily and quickly adapting to novel target datasets by fine-tuning. In this subsection, we evaluate the performance of transfer learning for F-SAM. We use a Deit-small model \cite{deit} pre-trained on ImageNet (with public available checkpoints\footnote{\url{https://github.com/facebookresearch/deit}}). We use AdamW \cite{loshchilov2017decoupled} as base optimizer and train the model for 10 epochs with batch size $128$, weight decay $10^{-5}$ and initial learning rate of $10^{-4}$. 
We adopt $\rho=0.075$ for SAM and F-SAM.
The results are shown in \cref{tab:finetuning}. We observe that F-SAM consistently outperforms SAM and AdamW by an accuracy gain of 0.1 to 0.7. 
An intriguing observation is that on small Flowers102 dataset \cite{nilsback2008automated} (with 1020 images), AdamW even outperforms SAM. 
We attribute this to the fact that when the dataset is small, the full gradient component within the minibatch gradient is more prominent, and its detrimental  impact on convergence and generalization becomes more evident.



\begin{table}[htbp]
    \centering
    \small
    \vspace{-2mm}
    \begin{tabular}{c|cccc}
    \toprule
         Datasets  &AdamW &SAM   &F-SAM (ours) \\
         \midrule
         \textbf{CIFAR-10}  &98.10$_{\pm 0.10}$ &98.27$_{\pm 0.05}$ &\textbf{98.43}$_{\pm 0.07}$\\
         \textbf{CIFAR-100}  &88.44$_{\pm 0.10}$ &89.10$_{\pm 0.11}$ &\textbf{89.49}$_{\pm 0.12}$\\
    \textbf{Standford Cars}  &74.78$_{\pm 0.09}$ &75.39$_{\pm 0.05}$ &\textbf{75.82}$_{\pm 0.14}$ \\
    \textbf{OxfordIIITPet}  &92.42$_{\pm 0.43}$ &92.70$_{\pm 0.26}$ &\textbf{92.90}$_{\pm 0.23}$ \\
         \textbf{Flowers102}  &74.82$_{\pm 0.36}$ &74.38$_{\pm 0.24}$  &$\textbf{75.15}_{\pm 0.35}$ \\
    \bottomrule 
    \end{tabular}
    \vspace{-3mm}
    \caption{Results on transfer learning by fine-tuning.}
    \vspace{-5mm}
    \label{tab:finetuning}
\end{table}

\subsection{Additional Studies}
\label{sec:ablation}
\subsubsection{Robustness to Label Noise}
Since previous works have shown that SAM is robust to label noise, in this subsection, we also test the performance of F-SAM in the presence of symmetric label noise by random flipping. The training settings are the same as in \cref{sec:trainingfromscratch}. 
From the results in \cref{tab:labelnoise}, we observe that F-SAM consistently improves the performance from SAM, confirming its improved generalization. Notably, such improvement is more obvious when the noise rate is large (i.e. 60\%-80\%). 
This is perhaps because when the data label is wrong, the harmfulness of increasing the sharpness of other batch data is more prominent, even posing challenges to convergence.
On severe noise ratio of 80\%, we find that vanilla SAM even suffers a collapse with the original setting, which is also reported by Foret et al.~\cite{foret2020sharpness}, and our F-SAM relieves it, achieving a remarkable accuracy gain of $28\%$.
\begin{table}[htbp]
	\setlength{\tabcolsep}{3.5pt}
    \centering
    \small
    \vspace{-3mm}
    \begin{tabular}{c|cccc}
    \toprule
        {Rates} &\textbf{20\%}  &\textbf{60\%} &\textbf{70\%} &\textbf{80\%}\\
    \midrule
    SGD &87.05$_{\pm 0.06}$ &52.21$_{\pm 0.23}$ &39.31$_{\pm 0.13}$ &27.86$_{\pm 0.53}$\\
    SAM &95.27$_{\pm 0.09}$ &90.08$_{\pm 0.09}$  &84.89$_{\pm 0.10}$ &31.73$_{\pm 3.37}$ \\
    F-SAM &\textbf{95.42}$_{\pm 0.08}$ &\textbf{90.47}$_{\pm 0.05}$ &\textbf{86.48}$_{\pm 0.07}$ &\textbf{59.39}$_{\pm 1.77}$ \\
    \bottomrule
    \end{tabular}
    \vspace{-3mm}
    \caption{Results under label noise on CIFAR-10 with ResNet-18.}
    \label{tab:labelnoise}
    \vspace{-3mm}
\end{table}
\vspace{-5mm}
\subsubsection{Robustness to Perturbation Radius}
One deficiency of SAM is its sensitivity to the perturbation radius, which may require choosing different radii for optimal performance on different datasets. Especially, an excessively large perturbation radius can result in a sharp decrease in generalization performance. We attribute a portion of this over-sensitivity to the impact of the adversarial perturbation derived from the current minibatch data on the other data samples. Such an impact can be more obvious when the perturbation radius increases as the magnitude of the full gradient component grows correspondingly. 
In our F-SAM, such impact on other data samples is eliminated as much as possible, and thus, we can expect that it is more robust to the choices of the perturbation radius. In \cref{fig:different-rho}, we plot the performance of SAM and F-SAM under different perturbation radii. We can observe that F-SAM is much less sensitive to $\rho$ than SAM. Especially on larger perturbation radius, the performance gain of F-SAM over SAM is more significant.
For example, when $\rho$ is set to 2x larger than the optimal on CIFAR100, SAM's performance suffers a sharp drop from $80.88\%$ to $78.83\%$, but F-SAM still maintains a good performance of $80.30\%$.
We further compare the training curves of SAM and F-SAM under different perturbation radii in Appendix~A6, showing that F-SAM greatly facilitates training on large radius.
This confirms F-SAM's improved robustness to different perturbation radii.

\subsubsection{Results on Different Batch Sizes}
As the main claim of this paper is that the full gradient component existing in SAM's adversarial perturbation does not contribute to generalization and may have detrimental effects.
As the training batch size increases, the corresponding full gradient component existing in the minibatch gradient can be strengthened. Thus, the improvement of F-SAM from removing the full gradient component could be more obvious. 
In \cref{fig:largebz}, we compare the performance of SGD, SAM, and our F-SAM under batch sizes from $\{ 128, 512, 2048 \}$ while keeping other hyper-parameters unchanged. 
We observe that the performance gain of F-SAM over SAM increases significantly as the batch size increases. This further validates our findings and supports the effectiveness of removing the full gradient component.

\noindent
\textbf{Connection to m-sharpness.} 
We would like to relate our removing the full gradient component to the intriguing $m$-sharpness in SAM, which is defined as the search for an adversarial perturbation that maximizes the sum of losses over sub-batches of $m$ training points \cite{foret2020sharpness}. This phenomenon emphasizes the importance of using a small $m$ to effectively improve generalization \cite{foret2020sharpness,andriushchenko2022towards, behdin2022improved}.
It is worth noting that by utilizing a small sub-batch, one implicitly suppresses the effects of the full gradient component. Hence, our findings can provide new insights into understanding $m$-sharpness.

\begin{figure}[!t]
    \centering
    \begin{subfigure}{0.45\linewidth}
 \centering
  \includegraphics[width=1\linewidth]{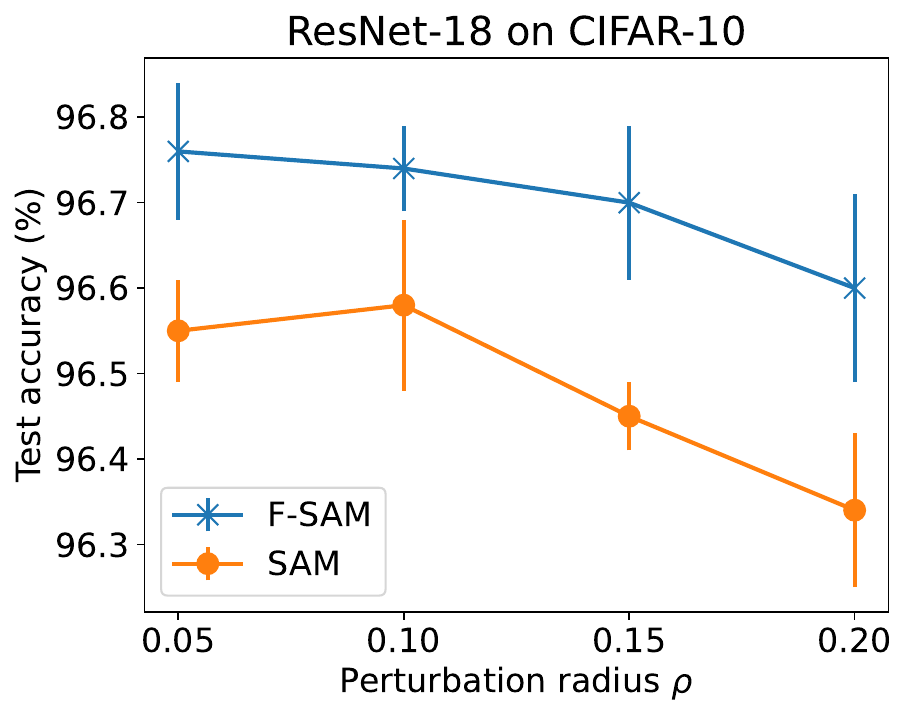}
 \end{subfigure}
    \hspace{0.03in}
    \begin{subfigure}{0.45\linewidth}
 \centering \hfill
  \includegraphics[width=1\linewidth]{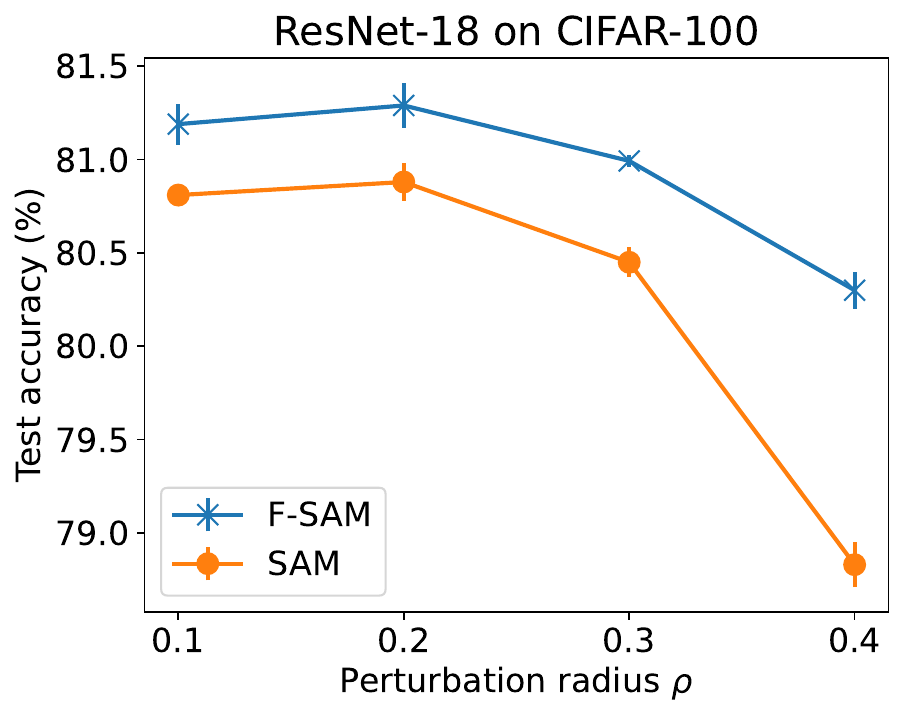}
 \end{subfigure}
 \vspace{-2mm}
    \caption{Results under different perturbation radii $\rho$. 
    }
    \label{fig:different-rho}
\end{figure}


\begin{figure}[!t]
\vspace{-3mm}
    \centering
\includegraphics[width=0.45\linewidth]{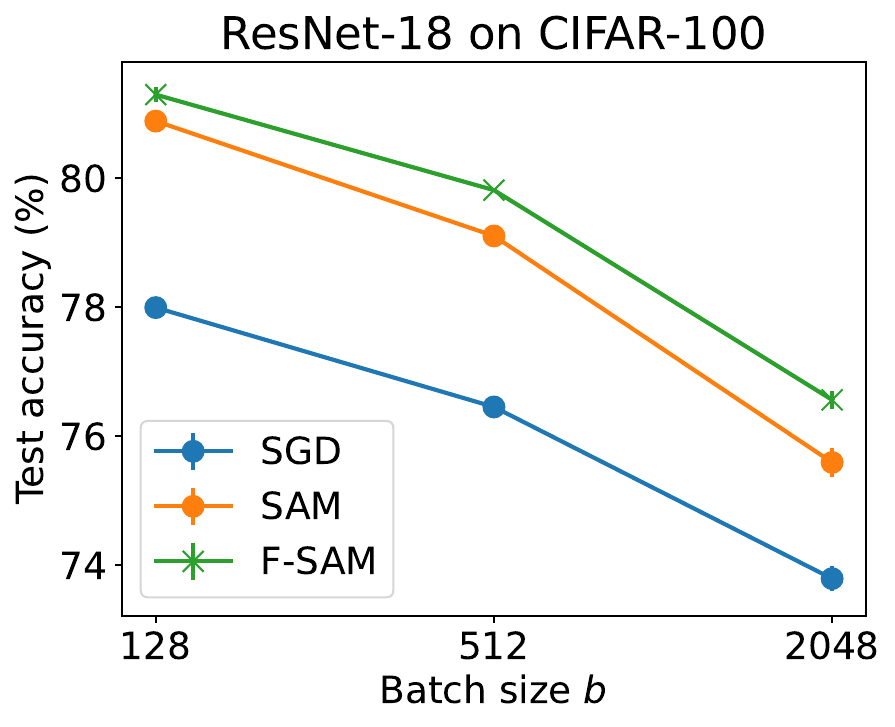}
    \vspace{-3mm}
    \caption{Performance comparison with different batch sizes.}
    \label{fig:largebz}
    \vspace{-5mm}
\end{figure}

\section{Conclusion}
In this paper, we conduct an in-depth  investigation into the core components of SAM's generalization.
By decomposing the minibatch gradient into two orthogonal components, we discover that the full gradient component in adversarial perturbation contributes minimally to generalization and may even have detrimental effects, while the stochastic gradient noise component plays a crucial role in enhancing generalization. 
We then propose a new variant of SAM called F-SAM to eliminate the undesirable effects of the full gradient component. Extensive experiments across various tasks demonstrate that F-SAM significantly improves the robustness and generalization performance of SAM.  
{
\vspace{-2mm}
\section*{Acknowledgement}
\vspace{-1mm}
This work was partially supported by National Key Research Development Project (2023YFF1104202), National Natural Science Foundation of China (62376155), Shanghai Municipal Science and Technology Research Program (22511105600) and Major Project (2021SHZDZX0102).
}

%% file: sec/X_suppl.tex
\clearpage
\setcounter{page}{1}
\setcounter{assumption}{1}
\onecolumn

\setcounter{section}{0}
\renewcommand{\thesection}{A\arabic{section}}
\renewcommand{\thecorollary}{A\arabic{corollary}}
\renewcommand{\thetheorem}{A\arabic{theorem}}
\renewcommand{\thefigure}{A\arabic{figure}}

\begin{center}
\Large
    \textbf{Friendly Sharpness-Aware Minimization} \\ \vspace{2mm}
    Supplementary Material \\ \vspace{2mm}
\end{center}


\section{Proof of Theorem 1
}
\label{sec:proof_theoremestimation}
The proof is based on \cite{kohler2017sub, staib2019escaping}.
We first introduce the Vector Bernstein inequality presented in \cite{kohler2017sub}. 
\begin{theorem}[Vector Bernstein]\label{thm:vecberstein}
 Let $\boldsymbol{x_1},\dots, \boldsymbol{x_n}$ be independent, zero-mean vector-valued random variables with common dimension d. We have the waeker Vector Berstein version as follows
\begin{equation}
    \begin{split}
        P\Big(\|\sum_{i=1}^n{\boldsymbol {x_i}}\|\ge \epsilon\Big) \leq \exp(-\frac{\epsilon^2}{8V} +\frac{1}{4}),
    \end{split}
\end{equation}
where $V = \sum_{i =1}^n \mathbb E\Big[{\|\boldsymbol{x_i}\| }^2\Big]$ is the sum of the variances of the centered vectors $\boldsymbol{x_i}$.
\end{theorem}
\begin{corollary}\label{cor:weightedberstein}
Let $\boldsymbol{x_1},\dots, \boldsymbol{x_n}$ be independent, zero-mean vector-valued random variables with common dimension d. Assume $\mathbb E\Big[{\|\boldsymbol {x_i}\| }^2\Big] \leq M$. Let $w\in\Delta_n$ in the simplex. Then we have
\begin{equation}
    \begin{split}
        P\Big(\|\sum_{i=1}^n{w_i \boldsymbol {x_i}}\|\ge \epsilon\Big) \leq \exp(-\frac{\epsilon^2}{8M\|w\|_2^2} +\frac{1}{4}).
    \end{split}
\end{equation}
\end{corollary}
\begin{proof}
    Simply apply Theorem~\ref{thm:vecberstein} with $\boldsymbol{\hat x_i} = w_i\boldsymbol{x_i}$.
\end{proof}
\noindent
Now we can apply the above concentration results to prove Theorem 1.
\begin{proof}
First we separately bound the bias and variance, then use Corollary~\ref{cor:weightedberstein}. The bias is:
\begin{align}
\left\|\sum_{t=1}^T w_t f\left(x_t\right)-f\left(x_T\right)\right\| & =\left\|\sum_{t=1}^T w_t\left(f\left(x_t\right)-f\left(x_T\right)\right)\right\| \\
& \leq \sum_{t=1}^T w_t\left\|f\left(x_t\right)-f\left(x_T\right)\right\| \\
& \leq \beta \sum_{t=1}^T w_t\left\|x_t-x_T\right\| \\
& \leq \beta \sum_{t=1}^T w_t \sum_{s=t+1}^T\left\|x_s-x_{s-1}\right\| \\
& \leq \gamma G \beta \sum_{t=1}^T w_t(T-t) \\
& =\gamma G \beta \cdot \frac{1}{\sum_{t=1}^T \lambda^{T-t}} \cdot \sum_{t=1}^T \lambda^{T-t}(T-t) .
\end{align}
\end{proof}
Note that by a well-known identity,
\begin{align}
\sum_{t=1}^T \lambda^{T-t}(T-t)=\sum_{s=0}^{T-1} s \lambda^s \leq \sum_{s=0}^{\infty} s \lambda^s=\frac{\lambda}{(1-\lambda)^2} .
\end{align}
Hence, the bias is bounded by
\begin{align}
\gamma G \beta \cdot \frac{1}{\sum_{t=1}^T \lambda^{T-t}} \cdot \frac{\lambda}{(1-\lambda)^2} & =\gamma G \beta \cdot \frac{1-\lambda}{1-\lambda^T} \cdot \frac{\lambda}{(1-\lambda)^2} \\
& =\gamma G \beta \cdot \frac{1}{1-\lambda^T} \cdot \frac{\lambda}{1-\lambda} \\
& \leq G \beta \cdot \frac{\gamma}{(1-\lambda)\left(1-\lambda^T\right)}.
\end{align}
Applying Corollary \ref{cor:weightedberstein} to $x_t=x_t-f\left(x_t\right)$, we have that
\begin{align}
\mathbb{P}\left(\left\|\sum_{t=1}^T w_t\left(y_t-f\left(x_t\right)\right)\right\|>k\right) \leq  \exp(-\frac{\epsilon^2}{8M^2\|w\|_2^2} +\frac{1}{4}).
\end{align}
Now note that
\begin{align}
\|w\|_2^2=\sum_{t=1}^T w_t^2 & =\frac{1}{\left(\sum_{t=1}^T\lambda^{T-t}\right)^2} \sum_{t=1}^T\left(\lambda^2\right)^{T-t} \\
& =\frac{(1-\lambda)^2}{\left(1-\lambda^T\right)^2} \sum_{t=1}^T\left(\lambda^2\right)^{T-t} \\
& =\frac{(1-\lambda)^2}{\left(1-\lambda^T\right)^2} \cdot \frac{1-\lambda^{2 T}}{1-\lambda^2} \\
& =\frac{1-\lambda^{2 T}}{\left(1-\lambda^T\right)^2} \cdot \frac{(1-\lambda)^2}{1-\lambda^2} \\
& =\frac{1+\lambda^T}{1-\lambda^T} \cdot \frac{1-\lambda}{1+\lambda} \\
& \leq \frac{2(1-\lambda)}{1-\lambda^T} .
\end{align}
Setting the right hand side of the high probability bound to $\delta$, we have concentration w.p. $1-\delta$ for $k$ satisfying
\begin{align}
\delta \ge \exp \left(-\frac{\epsilon^2}{8M{\|w\|}_2^2} +\frac{1}{4}\right).
\end{align}
Rearranging, we find
\begin{align}
\log \left(e^{-\frac{1}{4}}/{\delta}\right) \le \frac{\epsilon^2}{8M\|w\|_2^2}\\ 
\Leftrightarrow \epsilon \ge 2 \sqrt{2} \sqrt M\|w\|_2 \cdot \sqrt{\log \left(e^{-\frac{1}{4}} / \delta\right)}.
\end{align}
Combining this with the triangle inequality,
\begin{align}
\left\|\sum_{t=1}^T w_t y_t-f\left(x_T\right)\right\| & \leq\left\|\sum_{t=1}^T w_t y_t-\sum_{t=1}^T w_t f\left(x_t\right)\right\|+\left\|\sum_{t=1}^T w_t\left(f\left(x_t\right)-f\left(x_T\right)\right)\right\| \\
& \leq 2\sqrt{2} \sqrt M\|w\|_2 \sqrt{\log (e^{-\frac{1}{4}}/{\delta})}+G \beta \cdot \frac{\gamma}{(1-\lambda)\left(1-\lambda^T\right)} \\
&\leq 2\sqrt{2} \sqrt{\frac{2(1-\lambda)}{1-\lambda^T}} \sqrt M\sqrt{\log (e^{-\frac{1}{4}}/{\delta})}+G \beta \cdot \frac{\gamma}{(1-\lambda)\left(1-\lambda^T\right)} \\
& \leq 4 \sqrt M \frac{\sqrt{1-\lambda}}{\sqrt{1-\lambda^T}} \sqrt{\log (e^{-\frac{1}{4}} / \delta)}+G \beta \cdot \frac{\gamma}{(1-\lambda)\left(1-\lambda^T\right)},
\end{align}
with probability $1-\delta$. Since $1 / \sqrt{1-\lambda^T} \leq 1 /\left(1-\lambda^T\right)$, this can further be bounded by
\begin{align}
\left(4 \sqrt M\cdot \sqrt{1-\lambda} \cdot \sqrt{\log \left(e^{-\frac{1}{4}}/ \delta\right)}+G \beta \cdot \frac{\gamma}{1-\lambda}\right) \cdot \frac{1}{1-\lambda^T}.
\end{align}
Write $\alpha=1-\lambda$. The inner part of the bound is optimized when
\begin{align}
4 \sqrt M \cdot \sqrt{\alpha} \cdot \sqrt{\log \left(e^{-\frac{1}{4}} /\delta\right)} & =G \beta \cdot \frac{\gamma}{\alpha} \\
\Leftrightarrow \alpha^{3/2} & =\frac{G \beta \gamma}{4 \sqrt M \sqrt{\log (e^{-\frac{1}{4}}/ \delta)}} \\
\Leftrightarrow \alpha & = \frac{G^{2/3} \beta^{2/3} \gamma^{2/3}}{4 ^{2/3} \cdot M^{1/3} \cdot{\left(\log \left(e^{-\frac{1}{4}} / \delta\right)\right)}^{1/3}},
\end{align}
for which the overall inner bound is 
\begin{align}
2G\beta \cdot \frac{\gamma}{\alpha}=2^{7/3} \cdot M^{1/3} \cdot \log \left(e^{-\frac{1}{4}} / \delta\right)^{1/3} \cdot(G \beta \gamma)^{1/3}.
\end{align}
If $T$ is sufficiently large, the $1 /\left(1-\lambda^T\right)$ term will be less than 2 . In particular,
$$
T>\frac{2}{\log (1+\alpha)} \Longrightarrow \frac{1}{1-(1-\alpha)^T}<2 .
$$
Since $\log (1+\alpha)>\alpha / 2$ for $\alpha<1$, it suffices to have $T>4 / \alpha$.

\section{Proof of Lemma 1
}
\label{sec:proof_thm2}
\begin{proof}
\begin{align}
 			\mathbb{E} \left \langle \widetilde{\mathrm{Proj}}^\top_{\nabla L(\boldsymbol{w}_t)} \nabla L_{\mathcal{B}}(\boldsymbol{w}_t), \nabla L(\boldsymbol{w}_t) \right \rangle
 			=&~\mathbb{E} \left \langle \nabla L_{\mathcal{B}}(\boldsymbol{w}_t)  - \sigma \boldsymbol{m}_t, \nabla L(\boldsymbol{w}_t) \right \rangle \\
 			=&~ \| \nabla L(\boldsymbol{w}_t) \|^2 (1-\sigma). \label{equ:unbaised}
\end{align}
Eqn.~(\ref{equ:unbaised}) uses the assumption that $\boldsymbol{m}_t$ is an unbiased estimator to $\nabla L(\boldsymbol{w}_t)$. Then with $\sigma=1$, we complete the proof.
\end{proof}

\section{Proof of Theorem 2
}
\begin{proof}
    Denote $\boldsymbol{w}_{t+1/2}= \boldsymbol{w}_t+\rho \frac{\nabla L_{\mathcal{B}}(\boldsymbol{w}_{t})-\boldsymbol{m}_t}{\| \nabla L_{\mathcal{B}}(\boldsymbol{w}_{t})-\boldsymbol{m}_t \|}$.
    From Assumption 1,
     it follows that
    \begin{align}
    L(\boldsymbol{w}_{t+1})\le& L(\boldsymbol{w}_{t}) + \nabla L(\boldsymbol{w}_{t})^\top (\boldsymbol{w}_{t+1}-\boldsymbol{w}_{t})+\frac{\beta}{2}  \| \boldsymbol{w}_{t+1}- \boldsymbol{w}_{t} \|^2 
    \\=& L(\boldsymbol{w}_{t})-\gamma_t \nabla L(\boldsymbol{w}_{t})^\top \nabla L_{\mathcal{B}}(\boldsymbol{w}_{t+1/2})+\frac{\gamma_t^2 \beta}{2} \|  \nabla L_{\mathcal{B}}(\boldsymbol{w}_{t+1/2}) \|^2 
    \\=& L(\boldsymbol{w}_{t})-\gamma_t \nabla L(\boldsymbol{w}_{t})^\top \nabla L_{\mathcal{B}}(\boldsymbol{w}_{t+1/2}) \nonumber
    \\&+\frac{\gamma_t^2 \beta}{2} \left( \| 
     \nabla L_{\mathcal{B}}(\boldsymbol{w}_{t+1/2})
     -\nabla L(\boldsymbol{w}_{t}) \|^2- \| \nabla L(\boldsymbol{w}_{t}) \|^2 + 2\nabla L(\boldsymbol{w}_{t})^\top \nabla L_{\mathcal{B}}(\boldsymbol{w}_{t+1/2}) \right )
    \\=& L(\boldsymbol{w}_{t})-\frac{\gamma_t^2 \beta}{2} \| \nabla L(\boldsymbol{w}_{t}) \|^2+ \frac{\gamma_t^2 \beta}{2} \|  \nabla L_{\mathcal{B}}(\boldsymbol{w}_{t+1/2})-\nabla L(\boldsymbol{w}_{t}) \|^2-(1-\beta \gamma_t)\gamma_t \nabla L(\boldsymbol{w}_{t})^\top \nabla L_{\mathcal{B}}(\boldsymbol{w}_{t+1/2})
    \\\le&  L(\boldsymbol{w}_{t})-\frac{\gamma_t^2 \beta}{2} \| \nabla L(\boldsymbol{w}_{t}) \|^2+\gamma_t^2 \beta \|  \nabla L_{\mathcal{B}}(\boldsymbol{w}_{t+1/2})-\nabla L(\boldsymbol{w}_{t+1/2}) \|^2 \nonumber
    \\&+ \gamma_t^2 \beta \| \nabla L(\boldsymbol{w}_{t+1/2}) - \nabla L(\boldsymbol{w}_{t})\|^2-(1-\beta \gamma_t)\gamma_t \nabla L(\boldsymbol{w}_{t})^\top \nabla L_{\mathcal{B}}(\boldsymbol{w}_{t+1/2}).
\end{align}
The last step using the fact that $\|a-b\|^2\le 2\|a-c\|^2+2\|c-b\|^2$. Then taking the expectation on both sides gives:
\begin{align}
    \mathbb{E} [L(\boldsymbol{w}_{t+1}) ]\le \mathbb{E} [L(\boldsymbol{w}_{t})]-\frac{\gamma_t^2 \beta}{2} \mathbb{E} \| \nabla L(\boldsymbol{w}_{t}) \|^2 + \gamma_t^2 \beta M
    + \rho_t^2 \gamma_t^2 \beta^3 - (1-\beta \gamma_t)\gamma_t \mathbb{E} [  \nabla L(\boldsymbol{w}_{t})^\top \nabla L_{\mathcal{B}}(\boldsymbol{w}_{t+1/2}) ].
    \label{equ:descent}
\end{align}
For the last term, we have:
\begin{align}
\begin{split}
    \mathbb{E}  [  \nabla L(\boldsymbol{w}_{t})^\top \nabla L_{\mathcal{B}}(\boldsymbol{w}_{t+1/2}) ]
    =&\mathbb{E} \left [  \nabla L(\boldsymbol{w}_{t})^\top \left(\nabla L_{\mathcal{B}}(\boldsymbol{w}_{t+1/2}) -  \nabla L_{\mathcal{B}}(\boldsymbol{w}_{t})  + \nabla L_{\mathcal{B}}(\boldsymbol{w}_{t}) \right )  \right ] \\
    =&\mathbb{E} \left [  \|\nabla L(\boldsymbol{w}_{t}) \|^2 \right ] + \mathcal{C}.
\end{split}
\label{equ:cross_term}
\end{align}
where $\mathcal{C}=\mathbb{E} \left[ \nabla L(\boldsymbol{w}_{t})^\top
\left (\nabla L_{\mathcal{B}}(\boldsymbol{w}_{t+1/2}) - \nabla L_{\mathcal{B}}(\boldsymbol{w}_{t}) \right ) 
 \right]$. Using the Cauchy-Schwarz inequality, we have
\begin{align}
\begin{split}
\mathcal{C}
\le  &\mathbb{E}
\left [
    \frac{1}{2} \|\nabla L(\boldsymbol{w}_{t}) \|^2 +  \frac{1}{2} \| \nabla L(\boldsymbol{w}_{t+1/2})-  \nabla L(\boldsymbol{w}_{t}) \|^2 \right ] \\
    \le  &\frac{1}{2}  \mathbb{E}  \|\nabla L(\boldsymbol{w}_{t}) \|^2 + \frac{\rho_t^2\beta^2}{2}.
\end{split}
\label{equ:C}
\end{align}
Plugging Eqn.~(\ref{equ:cross_term}) and (\ref{equ:C}) into Eqn.~(\ref{equ:descent}), we obtain:
\begin{align}
    \mathbb{E} [L(\boldsymbol{w}_{t+1}) ]\le &\mathbb{E} [L(\boldsymbol{w}_{t})]-\frac{\gamma_t^2 \beta}{2} \mathbb{E} \| \nabla L(\boldsymbol{w}_{t}) \|^2 + \gamma_t^2 \beta M + \rho_t^2 \gamma_t^2 \beta^3 - (1-\beta \gamma_t)\gamma_t   \mathbb{E}   \|\nabla L(\boldsymbol{w}_{t}) \|^2   \\&+ (1-\beta \gamma_t)\gamma_t \left ( \frac{1}{2}  \mathbb{E}   \|\nabla L(\boldsymbol{w}_{t}) \|^2 + \frac{\rho_t^2\beta^2}{2} \right ) \\
    \le  &\mathbb{E} [L(\boldsymbol{w}_{t})] -\frac{\gamma_t}{2} \mathbb{E}   \|\nabla L(\boldsymbol{w}_{t}) \|^2 + \gamma_t^2 \beta M +\frac{1}{2}\gamma_t\rho_t^2\beta^2(1+\beta\gamma_t).
\end{align}
Taking summation over $T$ iterations, we have:
\begin{align}
     \frac{\gamma_0}{2 \sqrt{T}} \sum_{t=1}^{T}  \mathbb{E}   \|\nabla L(\boldsymbol{w}_{t}) \|^2\le \mathbb{E} [L(\boldsymbol{w}_{0})] - \mathbb{E} [L(\boldsymbol{w}_{T})]+(\beta M+\frac{1}{2} \rho_0^2 \beta^3)\gamma_0^2 \sum_{t=1}^{T}  \frac{1}{T} + \frac{1}{2} \rho_0^2 \beta^2 \gamma_0 \sum_{t=1}^{T} \frac{1}{t}.
\end{align}
This gives 
\begin{align}
\frac{1}{{T}} \sum_{t=1}^{T}  \mathbb{E}   \|\nabla L(\boldsymbol{w}_{t}) \|^2&\le
\frac{2 \left( \mathbb{E} [L(\boldsymbol{w}_{0})] - \mathbb{E} [L(\boldsymbol{w}_{T})] \right)}{\gamma_0 \sqrt{T}}+\frac{2\beta M \gamma_0+ \rho_0^2 \beta^3\gamma_0}{\sqrt{T}}+\frac{\rho_0^2 \beta^2 \log T}{\sqrt{T}} \\
&\le
\frac{2 \left( \mathbb{E} [L(\boldsymbol{w}_{0}) - L(\boldsymbol{w}^{*})] \right)}{\gamma_0 \sqrt{T}}+\frac{2\beta M \gamma_0+ \rho_0^2 \beta^3\gamma_0}{\sqrt{T}}+
\frac{\rho_0^2 \beta^2 \log T}{\sqrt{T}} \\
&=\frac{2 \Delta}{\gamma_0 \sqrt{T}}+\frac{\Theta}{\sqrt{T}}+
\frac{\Pi\log T}{\sqrt{T}},
\end{align}
where $\boldsymbol{w}^*$ is the optimal solution, 
$\Delta=\mathbb{E} [L(\boldsymbol{w}_{0}) - L(\boldsymbol{w}^{*})], \Theta=2\beta M \gamma_0+ \rho_0^2 \beta^3\gamma_0$, and $\Pi=\rho_0^2 \beta^2$.
\end{proof}

\section{Extension to SAM's variants}

Since we only modify the perturbation of SAM, our modification can be straightforwardly extended into the SAM variants, such as ASAM \cite{kwon2021asam} and FiserSAM \cite{kim2022fisher}. For SAM variants, their min-max objectives can be written into a unified formulation:
\begin{align}
	\min_{\boldsymbol{w}} \max_{\| T_w^{-1} \boldsymbol{\epsilon} \| \le \rho} L(\boldsymbol{w}+\boldsymbol{\epsilon}),
	\label{equ:sam-variant}
\end{align}
where $T_w$ is a normalization operator, e.g., $T_w=\|\boldsymbol{w}\|$ for ASAM.
The inner maximization problem in Eq.~(\ref{equ:sam-variant}) can then be solved via first-order approximation as follows:
\begin{align}
	\boldsymbol{\epsilon}_s=\rho \frac{T_w^2 \nabla L_\mathcal{B}(\boldsymbol{w})}
	{\|T_w \nabla L_\mathcal{B}(\boldsymbol{w})\|_2}.
	\label{equ:perturbation-sam-variant}
\end{align}
To incorporate our improvement, we can modify Eqn. (\ref{equ:perturbation-sam-variant}) as follows:
\begin{align}
	\boldsymbol{\epsilon}_s=\rho \frac{T_w^2 \left (\nabla L_\mathcal{B}(\boldsymbol{w}_t) - \sigma \boldsymbol{m}_t \right )}
	{\|T_w \left (\nabla L_\mathcal{B}(\boldsymbol{w}_t) - \sigma \boldsymbol{m}_t \right )\|_2}.
\end{align}

\section{Investigation Details}
\label{sec:more_investigation_details}

\subsection{Experimental Settings}
We follow the training setting of our main experiments.
\paragraph{Training from scratch.} 
 We train the models for 200 epochs and set the initial learning rate as 0.05 with a cosine learning rate schedule. The momentum and weight decay are set to 0.9 and 0.0005 for SGD, respectively. SAM adopt the same setting except that the weight decay is set to 0.001 following \cite{mi2022make,li2023enhancing}.
We apply standard random horizontal flipping, cropping, normalization, and cutout augmentation \cite{devries2017improved}. For SAM and its modified variants, we set the perturbation radius $\rho$ as 0.1 and 0.2 for CIFAR-10 and CIFAR-100 \cite{mi2022make,li2023enhancing}. 

\paragraph{Transfer learning.} 
We use a Deit-small model \cite{deit} pre-trained on ImageNet. We use AdamW \cite{loshchilov2017decoupled} as base optimizer and train the model for 10 epochs with batch size $128$, weight decay $10^{-5}$ and initial learning rate of $10^{-4}$. 
We adopt $\rho=0.075$ for SAM and its modified variants. We apply image resizing (to $224\times 224$) and normalization for data preprocessing without extra augmentations. 

\paragraph{SAM's modified variants.} 1) SAM-full: we use full gradient $\nabla L(\boldsymbol{w})$ over the entire training dataset to calculate SAM's perturbation, i.e., $\boldsymbol{\epsilon}_s =\rho \frac{\nabla L(\boldsymbol{w})} {\|{\nabla L(\boldsymbol{w})}\|}$;
2) SAM-db: we use an extra random batch data $\mathcal{B}'$ to calculate SAM's perturbation, i.e.,  $\boldsymbol{\epsilon}_s =\rho \frac{\nabla L_{\mathcal{B}'}(\boldsymbol{w})} {\|{\nabla L_{\mathcal{B}'}(\boldsymbol{w})}\|}$
3) SAM-noise: we use residual projection direction w.r.t. the full gradient to calculate the perturbation, i.e., 
$\boldsymbol{\epsilon}_s = \rho \frac{\mathrm{Proj}^\top_{\nabla L(\boldsymbol{w})} \nabla L_{\mathcal{B}}(\boldsymbol{w})}{\|\mathrm{Proj}^\top_{\nabla L(\boldsymbol{w})} \nabla L_{\mathcal{B}}(\boldsymbol{w})\|}$. We align the gradient of the model parameters as a vector to perform gradient projection.

\subsection{More Experiments on Effects of Full Gradient Component}

To further substantiate the detrimental effects of strengthening the full gradient components, we conducted additional experiments on the CIFAR-100 dataset using the VGG16-BN architecture, as illustrated in Figure \cref{fig:strength_vgg}.

\begin{figure}
    \centering
    \includegraphics[width=0.5\linewidth]{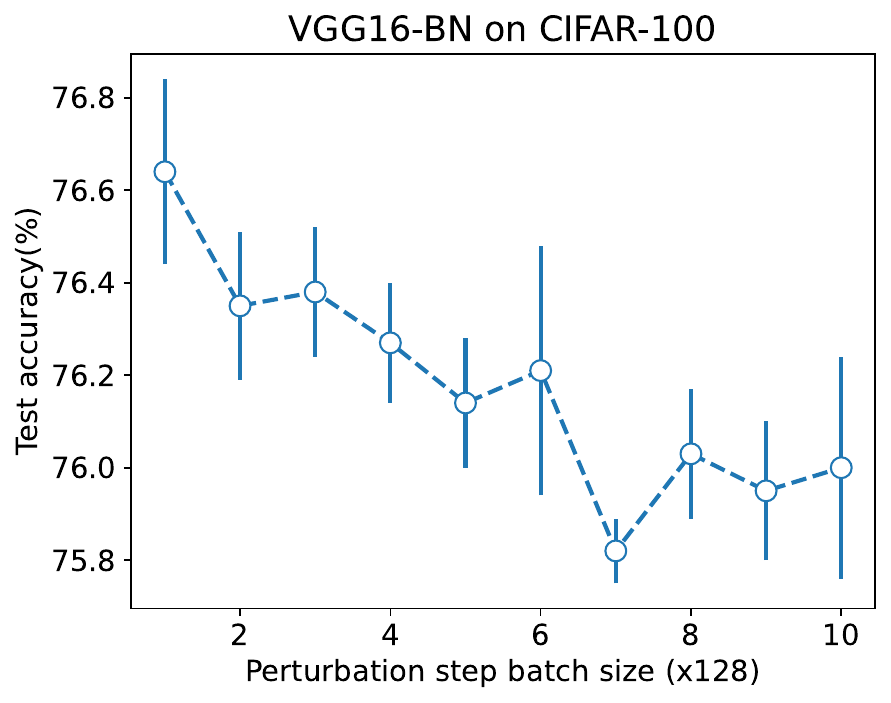}
    \caption{Results on enlarging the batch size of SAM's adversarial perturbation.}
    \label{fig:strength_vgg}
\end{figure}

\section{Training curves}
\label{sec:pratical_training_curves}
In \cref{fig:training_curves}, we compare the training curves of SAM and F-SAM. We observe that F-SAM achieves a faster convergence than SAM especially on the initial stage. This is because at the initial stage, the proportion of the full gradient component in the minibatch gradient is more significant,and hence removing this component to facilitate convergence in F-SAM has a more pronounced effect. Moreover, as the perturbation radius grows (2x in this case), the magnitude of the full gradient component in $\boldsymbol{\epsilon}_s$ also grows. 
This can significantly hinder the convergence of SAM and degrade its performance. In contrast, F-SAM is able to mitigate this undesired effects and maintain a good performance.

\begin{figure}[h]
    \centering
    \begin{subfigure}{0.45\linewidth}
 \centering
  \includegraphics[width=1\linewidth]{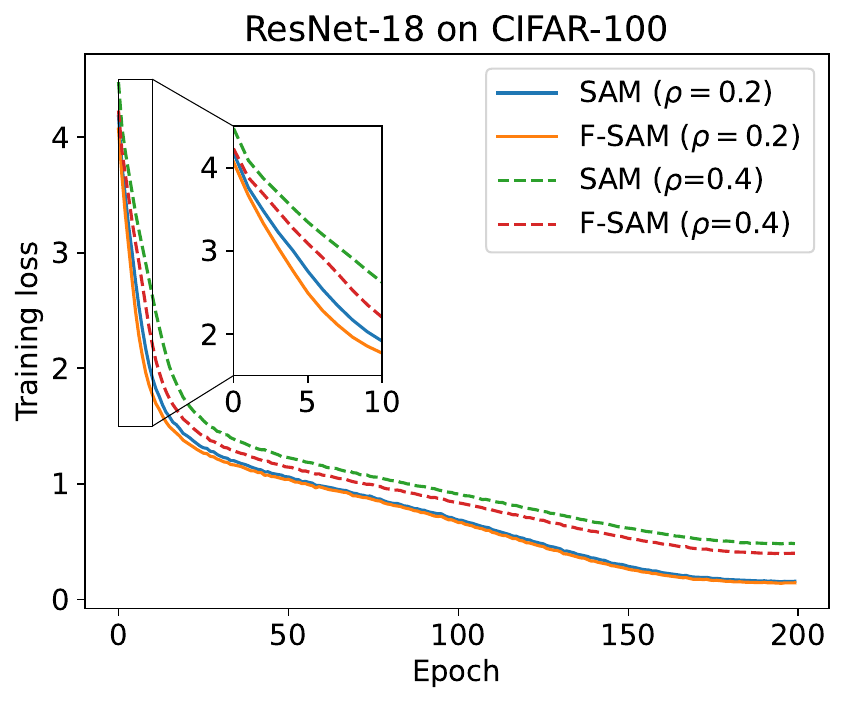}
 \end{subfigure}
    \hspace{0.03in}
    \begin{subfigure}{0.45\linewidth}
 \centering
  \includegraphics[width=1\linewidth]{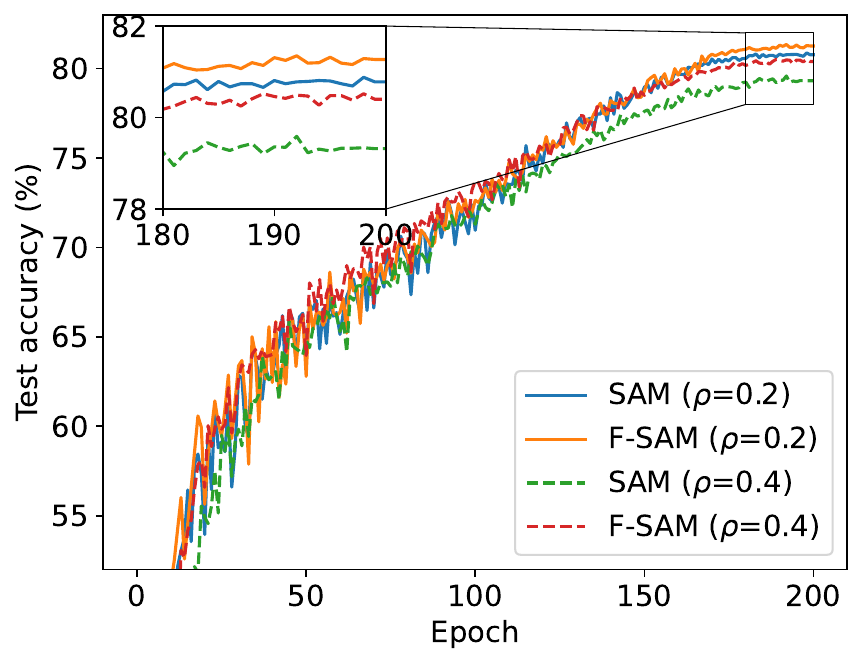}
 \end{subfigure}
 \vspace{-2mm}
    \caption{Training curve comparison on CIFAR-100 with ResNet-18. }
    \label{fig:training_curves}
\end{figure}

\section{Hessian Spectrum}
In \cref{fig:hessian}, we compare the Hessian eigenvalues of ResNet-18 trained with SAM and F-SAM. We focus on the largest eigenvalue $\lambda_1$ and the ratio of the largest to the fifth largest eigenvalue $\lambda_1/\lambda_5$.  We approximate the calculation for Hessian spectrum using the Lanczos algorithm \cite{pmlr-v97-ghorbani19b}. We observe that F-SAM achieves a smaller largest eigenvalue and smaller eigenvalue ratio compared with SAM.
This confirms that F-SAM converges to a flatter solution and achieves better generalization by removing the undesirable full gradient component in adversarial perturbation.
\begin{figure}[htbp]
	\centering
	\begin{subfigure}{0.45\linewidth}
		\centering
		\includegraphics[width=1\linewidth]{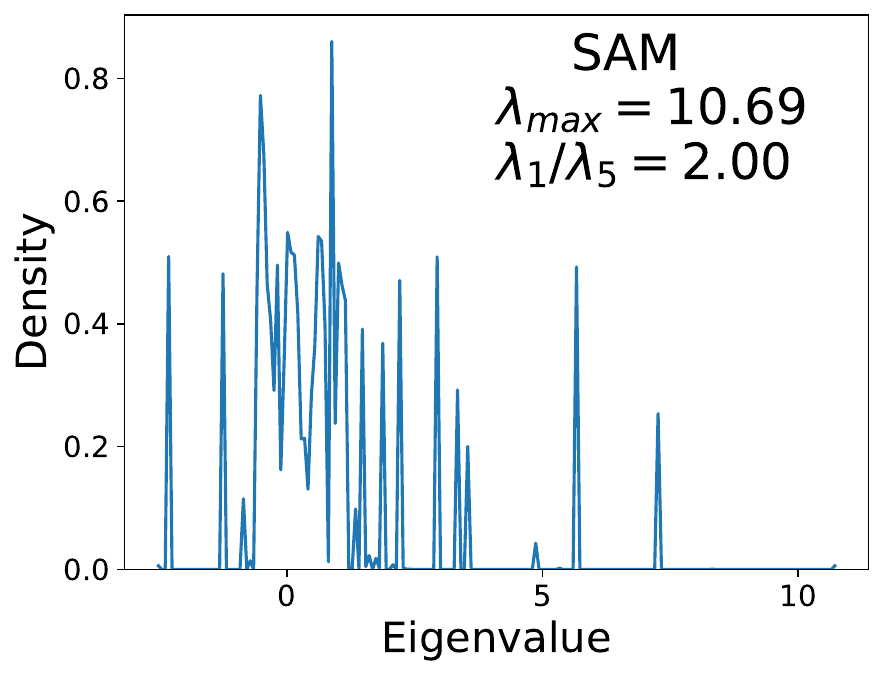}
		\label{fig:sam-fb}
	\end{subfigure}
    \hspace{0.03in}
	\begin{subfigure}{0.45\linewidth}
		\centering
		\includegraphics[width=1\linewidth]{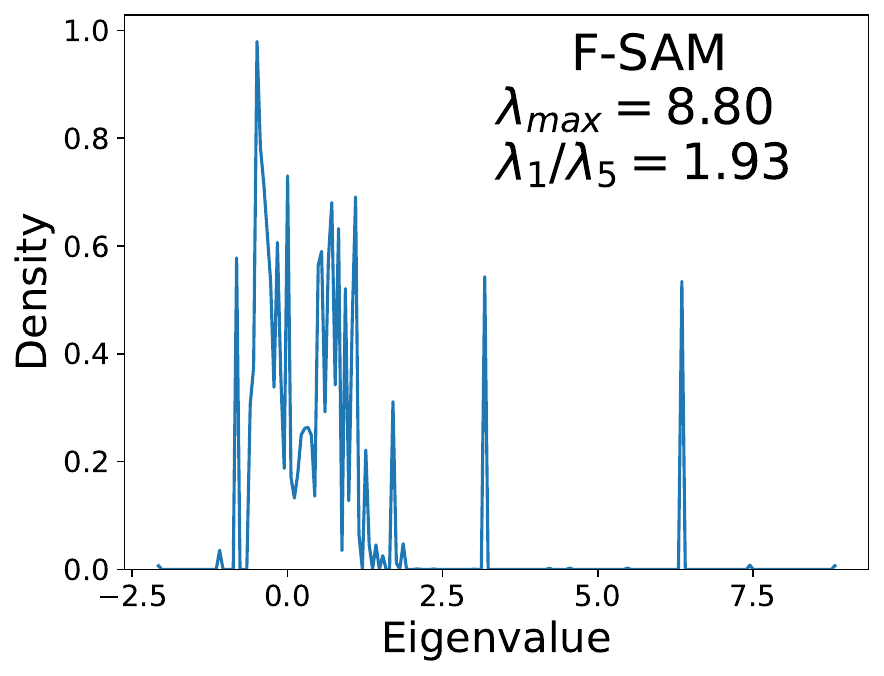}
		\label{fig:sam-db}
	\end{subfigure}
	\caption{Hessian spectrum comparison on CIFAR-10 with ResNet-18.} 
	\label{fig:hessian}
\end{figure}
